\documentclass[twoside,11pt]{article}

\usepackage{blindtext}

\usepackage[abbrvbib, preprint]{jmlr2e}

\usepackage{lastpage}

\ShortHeadings{Rethinking the Personalized Relaxed Initialization}{Rethinking the Personalized Relaxed Initialization}
\firstpageno{1}

\usepackage{algorithmic}
\usepackage{subfigure}
\usepackage{algorithm}
\usepackage{wrapfig}
\usepackage{amsmath}
\usepackage{amssymb}
\usepackage{booktabs}
\newtheorem{assumption}{\bf Assumption}
\allowdisplaybreaks[4]
\usepackage{multirow}

\begin{document}

\title{Rethinking the Personalized Relaxed Initialization in the Federated Learning: Consistency and Generalization}

\author{\name Li Shen \email mathshenli@gmail.com \\
        \addr School of Cyber Science and Technology \\
              Shenzhen Campus of Sun Yat-sen University
        \AND
        \name Yan Sun \email woodenchild95@outlook.com \\
        \addr School of Computer Science\\
              Faculty of Engineering \\
              The University of Sydney
        \AND
        \name Dacheng Tao \email dacheng.tao@ntu.edu.sg \\
        \addr Generative AI Lab \\
              College of Computing and Data Science \\
              Nanyang Technological University}


\maketitle

\begin{abstract}
Federated learning~(FL) is a distributed paradigm that coordinates massive local clients to collaboratively train a global model via stage-wise local training processes on the heterogeneous dataset.
Previous works have implicitly studied that FL suffers from the ``client-drift'' problem, which is caused by the inconsistent optimum across local clients. However, till now it still lacks solid theoretical analysis to explain the impact of this local inconsistency. 
To alleviate the negative impact of ``client drift'' and explore its substance in FL, in this paper, we first propose an efficient FL algorithm \textit{FedInit}, which allows employing the personalized relaxed initialization state at the beginning of each local training stage. Specifically, \textit{FedInit} initializes the local state by moving away from the current global state towards the reverse direction of the latest local state. Moreover, to further understand how inconsistency disrupts performance in FL, we introduce the excess risk analysis and study the divergence term to investigate the test error in FL. Our studies show that optimization error is not sensitive to this local inconsistency, while it mainly affects the generalization error bound. Extensive experiments are conducted to validate its efficiency. The proposed \textit{FedInit} method could achieve comparable results compared to several advanced benchmarks without any additional training or communication costs. Meanwhile, the stage-wise personalized relaxed initialization could also be incorporated into several current advanced algorithms to achieve higher generalization performance in the FL paradigm.
\end{abstract}

\section{Introduction}
\label{introduction}
Since \cite{mcmahan2017communication} developed federated learning~(FL), it has become a promising paradigm to effectively make full use of edged computational powers of local devices. \cite{MAL-083} further classify the different modes in FL based on the specific tasks and unique environmental setups. Different from centralized training, FL utilizes a central server to coordinate the clients to perform several local training stages and aggregate local models as one global model regularly. This creative training framework greatly improves the effective use of edged devices. However, it also has unavoidable disadvantages. Due to the local heterogeneous dataset, it usually suffers from the significant performance degradation.\\

\noindent Several previous studies explore the essence of performance limitations and inherent biases in FL, and summarize them as the ``client-drift'' problems~\citep{acar2021federated,karimireddy2020scaffold,li2020federated,sun2023fedspeed,wang2021local,xu2021fedcm,yang2021achieving}. From the global objective perspective, \cite{karimireddy2020scaffold} claim that the aggregated local optimum is always far away from the global optimum due to the local heterogeneity, which introduces the issue of $w^\star\neq\frac{1}{C}\sum_{i}w_i^\star$ in FL. However, when the local training iterations are limited, local clients could not exactly achieve their local optimum. To describe this negative impact more accurately, \cite{acar2021federated} and \cite{wang2021local} point out that each locally optimized objective should be regularized to be aligned with the global objective. Moreover, beyond the requirement of the local consistent objectives, \cite{xu2021fedcm} also indicates that the performance degradation could be mitigated when each local updates maintain high consistency at each communication round, which is more similar to the centralized scenarios. These explorations intuitively provide the forward-looking insights of improving the performance in FL. Although this analysis explains the inherent biases across local heterogeneous datasets, there is still no solid theoretical support to further understand the incomprehensible impact of the consistency, which also greatly hinders its developments. Therefore, our work tries to provide a new perspective on consistency.\\

\noindent {\bf Motivation.} From the training strategies, the largest difference between FL and general centralized training is to adopt local training processes, which results in local clients tilt towards the local optimum. We spontaneously and naturally expect to add an inverse effect in the local training process to correct the deviation. Several previous studies have introduced more variables, which has caused great pressure on communication bottlenecks. To further alleviate the negative impact of the ``client-drift'' problem and strengthen consistency in the FL paradigm without any additional communication burdens, in this paper, we take into account adopting the personalized relaxed initialization at the beginning of each communication round, dubbed \textit{FedInit} method. \\

\noindent Specifically, \textit{FedInit} initializes each local models of active clients by moving away from the current global state towards the reverse direction of the current latest local state, which is named relaxed initialization~(RI). Personalized RI helps each local model to revise its divergence and gather together with each other during the local training process. Because each initialization of the local model is more distant from the local optimum than the global model, locally trained solutions naturally maintain higher consistency under the same local training interval. This flexible approach is surprisingly effective in FL and only adopts one additional coefficient to control the divergence level of the initialization. It could also be easily incorporated as a plug-in into other advanced benchmarks to further improve their efficiency. We also provide a schematic illustration of the effects of this initialization.\\

\noindent Moreover, to explicitly understand how and why RI could help to improve the performance in FL, we introduce the excess risk analysis to investigate the test error of \textit{FedInit} under the smooth non-convex objectives, including the joint analysis on both the optimization error and the generalization error. Our theoretical analysis indicates that both optimization error and generalization error could be greatly improved by the RI technique. Furthermore, under \textit{P\L}-condition, RI could efficiently reduce the excess risk and the test error. Extensive empirical studies are conducted to validate the efficiency of the proposed \textit{FedInit} method on different classical federated experimental setups. On the CIFAR-10$/$100 dataset, it could achieve SOTA results compared to several advanced benchmarks without additional costs. It also helps to enhance the consistency level in FL.\\

\noindent Main contributions of our work are summarized as follows:
\begin{itemize}
    \item We propose an efficient and novel FL method, dubbed \textit{FedInit}, which adopts the personalized relaxed initialization~(RI) state on the selected local clients at each communication round. RI is dedicated to enhancing local consistency during training, and it is also a practical plug-in that could easily to incorporated into other methods.
    \item One important contribution is that we introduce the excess risk analysis in the proposed \textit{FedInit} method to understand the intrinsic impact of local consistency. Our theoretical studies prove that RI could help to improve both the optimization error and generalization error on the smooth and non-convex objectives. 
    \item Extensive numerical studies are conducted on the real-world dataset to validate the efficiency of the \textit{FedInit} method, which could achieve the comparable results of several SOTA benchmarks, without any additional communication costs during the training.
\end{itemize}

\section{Related Work}
\label{related_work}

\noindent\textbf{Improving Consistency in FL.} FL employs an enormous number of edge devices to jointly train a single model among the isolated heterogeneous dataset~\citep{MAL-083, mcmahan2017communication}. As a standard benchmark, \textit{FedAvg}~\citep{fedopt,mcmahan2017communication, yang2021achieving} allows the local stochastic gradient descent~(local SGD)~\citep{gorbunov2021local,lin2018don,woodworth2020minibatch} based updates and uniformly selected partial clients' participation to alleviate the communication bottleneck. The stage-wise local training processes lead to significant divergence for each client~\citep{inconsistency1,inconsistency2,inconsistency3,wang2021local}. To improve the efficiency of the FL paradigm, a series of methods are proposed. \cite{karimireddy2020scaffold} indicate that inconsistent local optimums cause the severe ``client drift'' problem and propose the \textit{SCAFFOLD} method which adopts the variance reduction~\citep{defazio2014saga,johnson2013accelerating} technique to mitigate it. \cite{li2020federated} penalize the prox-term on the local objectives to force each local update towards both the local optimum and the last global state. \cite{zhang2021fedpd} utilize the primal-dual method to improve consistency via solving local objectives under the equality constraint. Specifically, a series of works further explore and extend the alternating direction method of multipliers~(ADMM) to optimize the global objective~\citep{acar2021federated,gong2022fedadmm,wang2022fedadmm,zhou2023federated,sun2024fedpd}, which could also enhance the consistency term. Beyond these, a series of momentum-based methods are proposed to strengthen local consistency. \cite{slowmo} study a global momentum update method to stabilize the global model. Further, \cite{feddc} adopt the local drift correction via a momentum-based variable to revise the local gradients, efficiently reducing inconsistency. \cite{fedadc,xu2021fedcm,sun2023efficient} propose a similar client-level momentum to force the local update towards the last global direction. Several variants of client-level momentum methods adopt the inertial momentum to further improve the local consistency level~\citep{liu2023enhance,fedproto}. Improving the consistency in FL remains a very important and promising research direction. Though these studies involve a lot of heuristic discussions, exploring and improving local consistency is always one of the essential studies in FL.\\

\noindent\textbf{Generalization Efficiency in FL.} Several works have studied the properties of the generalization errors in FL. Based on the margin loss~\citep{bartlett2017spectrally,farnia2018generalizable,neyshabur2017pac}, \cite{reisizadeh2020robust} develop a robust FL paradigm to alleviate the distribution shifts across the heterogeneous clients. \cite{shi2021fed} study the efficient and stable model technique of model ensembling. \cite{yagli2020information} prove the information-theoretic bounds on the generalization error and privacy leakage in the general FL paradigm. \cite{qu2022generalized} propose to adopt the sharpness aware minimization~(SAM) optimizer on the local client to enhance the flatnesses of the loss landscape. \cite{improving_sam,sun2023dynamic,sun2023fedspeed,shi2023make,shi2023improving} propose several variants based on SAM that could achieve higher performance. Beyond the optimizer, \cite{caldarola2023window} further propose a window-based averaging to improve the generalization efficiency in FL. Different from these works, this paper focuses both on the optimization and generalization performance and explores the impact of consistency items in FL.

\section{Methodology}
\label{methodology}
\subsection{Preliminaries}
We first introduce the notations in our paper. We use italics for scalars and $[\cdot]$ denotes the integer list within $[1,\cdot]$. Unless otherwise specified, all four arithmetic operators conform to element-wise operations. Table~\ref{notation} shows some special notations. Other symbols are defined when they are first introduced.\\
\begin{table}[t]
    \caption{Notation tables.}
    \vspace{0.1cm}
    \label{notation}
    \centering
    \begin{tabular}{c|c}
        \toprule
        Symbol  &  Definition \\
        \midrule
        $\mathcal{N}, N$  &  set\ /\ number of active clients \\
        $\mathcal{C}, C$  &  set\ /\ number of total clients \\
        $K,k$  &   number\ /\ index of local interval \\
        $T,t$  &   number\ /\ index of communication rounds \\
        $w, w^\star$    &   model parameters\ /\ optimum parameters \\
        $\mathcal{S},S$  & local dataset\ /\ size of local dataset \\
        $\Delta^t$  &  consistency\ /\ divergence term at $t$-round \\
        $\mathcal{O}(\cdot)$ & the order of approximation\\
        $\widetilde{\mathcal{O}}(\cdot)$ & the order of approximation without Logarithmic term\\
        \bottomrule
    \end{tabular}
\end{table}

Following the previous study~\citep{sun2023understanding}, there are a very large number of local clients to collaboratively train a global model. Due to privacy protection and unreliable network bandwidth, only a fraction of devices are open-accessed at any one time~\citep{MAL-083,qu2022generalized}. Therefore, we define each client stores a private dataset $\mathcal{S}_i=\left\{z_j\right\}$ where $z_j$ is drawn from an unknown unique distribution $\mathcal{D}_i$. The whole local clients constitute a set $\mathcal{C}=\{i\}$ where $i$ is the index of each local client and $\vert\mathcal{C}\vert=C$. Actually, in the training process, we expect to approach the optimum of the population risk:
\begin{equation}
    w_{\mathcal{D}}^\star\in\arg\min_{w}\left\{ F(w)\triangleq \frac{1}{C}\sum_{i\in\mathcal{C}}F_i(w)\right\},
\end{equation}
where $F_i(w)=\mathbb{E}_{z_j\sim\mathcal{D}_i}F_i(w,z_j)$ is the local population risk.
While in practice, we usually consider the empirical risk minimization of the non-convex finite-sum problem in FL as:
\begin{equation}
     w^\star\in\arg\min_w \left\{ f(w) \triangleq \frac{1}{C}\sum_{i\in\mathcal{C}}f_i(w)\right\},
\end{equation}
where $f_i(w)=\frac{1}{S_i}\sum_{z_j\in\mathcal{S}_i}f_i(w;z_j)$ is the local empirical risk. In Section~\ref{excess_risk}, we will analyze the difference between these two results. Furthermore, we introduce the excess risk analysis to upper bound the test error and further understand how to improve the efficiency in FL.

\subsection{Personalized Relaxed Initialization}

In this part, we introduce the relaxed initialization in \textit{FedInit} method. \textit{FedAvg} proposes the local-SGD-based implementation in the FL paradigm with a partial participation selection. It allows uniformly selecting a subset of clients $\mathcal{N}$ to participate in the current training. In each round, it initializes the local model as the last global model. Therefore, after each round, the local models are always far away from each other. The local offset $w_{i,K}^{t-1}-w^t$ is the main culprit leading to inconsistency. Moreover, for different clients, their impacts vary with local heterogeneity. To alleviate this divergence, we propose the \textit{FedInit} method which adopts the personalized relaxed initialization at the beginning of each round. Concretely, on the selected active clients, it begins the local training from a new personalized state, which moves away from the last global model towards the reverse direction from the latest local state~(Line.6 in Algorithm~\ref{algorithm}). A coefficient $\beta$ is adopted to control the level of personality. This offset $\beta(w^t - w_{i,K}^{t-1})$ in the relaxed initialization (RI) provides a correction that could help local models gather together after the local training process. Furthermore, this relaxed initialization is irrelevant to the local optimizer, which means, it could be easily incorporated into other methods. Additionally, \textit{FedInit} does not require extra auxiliary information to communicate between the server and local clients, which indicates it a practical technique.
\begin{algorithm}[tb]
	\renewcommand{\algorithmicrequire}{\textbf{Input:}}
	\renewcommand{\algorithmicensure}{\textbf{Output:}}
	\caption{{\it FedInit} Method}
	\begin{algorithmic}[1]
		\REQUIRE model $w$, local model $w_i$, $T$, $K$, $\beta$
		\ENSURE global parameters $w^{T}$
        \STATE Initialize $w^{-1}=w_{i,0}^{-1}=w^0$.
		\FOR{$t = 0, 1, 2, \cdots, T-1$}
		\STATE randomly select active clients set $\mathcal{N}$ from $\mathcal{C}$
		\FOR {client $i \in \mathcal{N}$ in parallel}
		\STATE send the $w^t$ to the active clients
		\STATE set the $w_{i,0}^t=w^t+\beta(w^t-w_{i,K}^{t-1})$
		\FOR{$k = 0, 1, 2, \cdots, K-1$}
		\STATE compute unbiased stochastic gradient $\mathbf{g}_{i,k}^{t}$
		\STATE $w_{i,k+1}^t = w_{i,k}^t - \eta g_{i,k}^t$
		\ENDFOR
		\STATE communicate $w_{i}^t=w_{i,K}^t$ to the server
		\ENDFOR
        \FOR {client $i \notin \mathcal{N}$ in parallel}
        \STATE $w_{i,K}^t = w_{i,K}^{t-1}$
        \ENDFOR
        \STATE $w^{t+1}=\frac{1}{N}\sum_{i\in\mathcal{N}}w_{i}^t$
		\ENDFOR
	\end{algorithmic}
	\label{algorithm}
\end{algorithm}

\section{Theoretical Analysis}
\label{theoretical_analysis}
In this section, we first introduce the theoretical analysis which could provide a comprehensive analysis of the joint performance of both optimization and generalization. In the second part, we introduce the main assumptions adopted in our proofs and discuss them in different situations. Then we demonstrate the main theorems in our analysis, including both the optimization and generalization respectively.

\subsection{Excess Risk Error}
\label{excess_risk}
Since \cite{karimireddy2020scaffold} indicated that the ``client-drift" problem seriously damages the performance in the FL paradigm, many previous works~\citep{huang2023fusion,karimi2021layer,karimireddy2020scaffold,reddi2020adaptive,sun2023adasam,wang2021local,xu2021fedcm,yang2021achieving} have widely investigated the inefficiency of consistency in the FL paradigm. However, the most of analyses focus on the studies from the onefold perspective of convergence but ignore learning its impact on generality. To further provide a comprehensive studies of the joint performance of both the optimization and generalization in FL, we introduce the well-known excess risk in the analysis.\\

\noindent We denote $w^T$ as the final model generated by \textit{FedInit} method after $T$ communication rounds. Compared with $f(w^T)$, we mainly focus on the efficiency of $F(w^T)$ which corresponds to its generalization performance and test accuracy. Thus, we analyze the $\mathbb{E}[F(w^T)]$ from the excess risk $\mathcal{E}_E$ as:
\begin{align}
         \mathcal{E}_E 
       = \mathbb{E}[F(w^T)] - \mathbb{E}[f(w^*)] \nonumber = \underbrace{\mathbb{E}[F(w^T) - f(w^T)]}_{\mathcal{E}_G:\ \textit{generalization error}} + \underbrace{\mathbb{E}[f(w^T) - f(w^*)]}_{\mathcal{E}_O:\ \textit{optimization error}}.
\end{align}
Generally, the $\mathbb{E}[f(w^*)]$ is expected to be very small and even to zero if the model could fit the dataset. Thus $\mathcal{E}_E$ could be considered as the joint efficiency of the generated model $w^T$. Thereinto, $\mathcal{E}_G$ means the different performance of $w^T$ between the training dataset and the test dataset, and $\mathcal{E}_O$ means the similarity between $w^T$ and optimization optimum $w^\star$. From the perspective of the excess risk, $\varepsilon_E$ approximates our focus $\mathbb{E}[F(w^T)]$. We investigate the optimization and generalization performance respectively in the following part. At last, we also provide a simple analysis of the divergence term.

\subsection{Assumptions}
In this part, we mainly introduce the assumptions adopted in our analysis. Then We discuss their properties and distinguish which proofs they are adopted in.
\begin{assumption}[L-smooth]
\label{assumption:smooth}
    Local function $f_i$ satisfies the $L$-smoothness, i.e., for $\forall w_1, w_2 \in \mathbb{R}^d$, $\Vert \nabla f_i(w_1) - \nabla f_i(w_2) \Vert \leq L\Vert w_1 - w_2\Vert$ for any data sample.
\end{assumption}
\begin{assumption}[{\it P\L}-condition]
\label{assumption:PL}
    Global function $f$ satisfies {\it P\L}-condition, i.e., for $\forall w \in \mathbb{R}^d$,  $ \mu\left(f(w)-f(w^\star)\right)\leq\Vert\nabla f(w)\Vert^2$, where $w^\star\in\arg\min_{w} f(w)$ is one optimum state.
\end{assumption}

Assumption~\ref{assumption:smooth} is widely adopted to analyze FL methods~\citep{karimireddy2020scaffold,yang2021achieving,sun2023dynamic,qu2022generalized,sun2023fedspeed}. Assumption~\ref{assumption:PL} is merely adopted in joint analysis of excess risk. Most analysis of general smooth non-convex objectives in FL only indicates $\frac{1}{T}\sum_t\mathbb{E}\Vert\nabla f(w^t)\Vert^2$ is necessarily bounded. To approximate the optimization error $\mathbb{E}\left[f(w^T)-f(w^\star)\right]$ in excess risk, we follow \cite{zhou2021towards} to adopt the {\it P\L}-condition. We also provide results without {\it P\L}-condition.

\begin{assumption}[Stochastic]
\label{assumption:stochastic}
    The unbiased stochastic gradients $g_i=\nabla f_{i}(w,z)$ of any data sample $z$ satisfies the bounded variance, i.e., for $\forall w \in \mathbb{R}^d$, $\mathbb{E}\Vert g_i - \nabla f_{i}(w)\Vert^2 \leq \sigma_l^2$.
\end{assumption}
\begin{assumption}[Local Interpolation]
\label{assumption:interpolation}
    The unbiased stochastic gradients $g_i=\nabla f_{i}(w,z)$ of any data sample $z$ satisfies the interpolations, i.e.,
    for $\forall w \in \mathbb{R}^d$, $\Vert g_i\Vert^2 \leq a^2\Vert\nabla f_{i}(w)\Vert^2$.
\end{assumption}

Assumption~\ref{assumption:stochastic} is widely adopted to analyze the stochastic optimization. Assumption~\ref{assumption:interpolation} is recently adopted to describe the property of over-parameterized models in deep learning~\citep{vaswani2019fast,karzand2019maximin,wang2022convergence}. Large models always show the potential strong ability to flawlessly handle each single data sample, which implies the separated gradient $\nabla f_i(w^\star,z)\rightarrow 0$ if $\nabla f_i(w^\star)\rightarrow 0$.

\begin{assumption}[Heterogeneity]
\label{assumption:heterogeneity}
    We consider the dissimilarity among different local clients to be bounded, i.e., for $\forall w \in \mathbb{R}^d$, $\mathbb{E}\Vert\nabla f_i(w) - \nabla f(w)\Vert^2\leq \sigma_g^2$.
\end{assumption}
\begin{assumption}[Global Interpolation]
\label{assumption:global_interpolation}
    Similarly, we take into account local gradients $\nabla f_i(w)$ maintain interpolations, i.e., for $\forall w \in \mathbb{R}^d$, $\Vert\nabla f_i(w)\Vert^2\leq b^2\Vert\nabla f(w)\Vert^2$.
\end{assumption}

Assumption~\ref{assumption:heterogeneity} is widely adopted to measure the dissimilarity of the private local dataset. Assumption~\ref{assumption:global_interpolation} is an extension of interpolations to the global model, which maintains a similar property as Assumption~\ref{assumption:interpolation} on the over-parameterized models. Assumption~\ref{assumption:interpolation} and \ref{assumption:global_interpolation} do not have to be established at the same time. For instance, even if local interpolations hold for the local dataset, global interpolations still could not be necessarily supported under large heterogeneity.

\begin{assumption}[Lipschitz Continuity]
\label{assumption:assumption_Lipschitz}
    The global function $f$ satisfies the $L_G$-Lipschitz property, i.e., for $\forall w_1, w_2 \in \mathbb{R}^d$, $  \Vert f(w_1) - f(w_2) \Vert \leq L_G\Vert w_1 - w_2\Vert$.
\end{assumption}

Assumption~\ref{assumption:assumption_Lipschitz} is a generally strong assumption which indicates the bounded gradient $\Vert\nabla f(w)\Vert\leq L_G$. Recent studies learn it may not always hold for the deep models~\citep{kim2021lipschitz,mai2021stability,patel2022gradient,das2023beyond}. The previous stability analysis always adopts it in the iterative formulation. However, we propose to remove this assumption in the iterative proofs. We only adopt it at the last state, which implies $\Vert\nabla f(w^T)\Vert\leq L_G$ when $T$ is sufficiently large.

\subsection{Optimization Analysis}
\begin{table*}[t]
  \caption{Summary of the convergence bound of our proposed \textit{FedInit} method.}
  \label{table:optimization}
  \centering
  \vspace{0.1cm}
  \begin{tabular}{c|ccc}
    \toprule
    Assumption & Convergence bound & Convergence rate \\
    \midrule    
    \ref{assumption:smooth}, \ref{assumption:stochastic}, \ref{assumption:heterogeneity} & $\frac{1}{T}\sum_{t=0}^{T-1}\mathbb{E}\Vert\nabla f(w^t)\Vert^2\leq \frac{D}{\lambda\eta KT} + \eta\frac{\kappa L\sigma}{\lambda N} - \frac{13\beta^2\kappa_\beta L^2}{\lambda\eta NKT}\Delta^T$ & $\mathcal{O}\left(\frac{1}{\sqrt{NKT}}\right)$ \\
    \midrule
    \ref{assumption:smooth}, \ref{assumption:PL}, \ref{assumption:stochastic}, \ref{assumption:heterogeneity} & $\mathbb{E}\left[f(w^T)-f(w^\star)\right]\leq e^{-\lambda\mu\eta KT}D + \eta\frac{\kappa L\sigma}{\lambda\mu N} + \mathcal{O}\left(\frac{\eta}{T}+\eta^2\right)$ & $\mathcal{O}\left(\frac{1}{NKT}\right)$\\
    
    \midrule
    \ref{assumption:smooth}, \ref{assumption:interpolation}, \ref{assumption:global_interpolation} & $\frac{1}{T}\sum_{t=0}^{T-1}\mathbb{E}\Vert\nabla f(w^t)\Vert^2\leq\frac{D}{\zeta\eta KT} - \frac{2\gamma_\beta\beta^2L}{\zeta\eta KT}\Delta^T$ & $\mathcal{O}\left(\frac{1}{NKT}\right)$\\
    \midrule
    \ref{assumption:smooth}, \ref{assumption:PL}, \ref{assumption:interpolation}, \ref{assumption:global_interpolation} & $\mathbb{E}\left[f(w^T)-f(w^\star)\right]\leq e^{-\zeta\mu\eta KT}D + R_\beta\eta^2K^2LD$ & $\mathcal{O}\left(\frac{1}{T^2}\right)$\\
    \bottomrule
  \end{tabular}
\end{table*}
In this part, we predominantly demonstrate the optimization error and the convergence rate of our \textit{FedInit} method under two gradient properties~(bounded variance and interpolation) and two function properties~(smoothness and {\it P\L}-condition), respectively. We also discuss the impact of the consistency. All detailed proofs could be referred to the Appendix~\ref{app:opt}.

\begin{theorem}
\label{thm1}
    Under Assumption~\ref{assumption:smooth}, \ref{assumption:stochastic}, and \ref{assumption:heterogeneity}, let participation ratio be $N/C$ where $1<N<C$, let the learning rate satisfies $\eta\leq 1/NKL$ where $K > 1$, and let the RI coefficient be a small positive constant, after total $T$ communication rounds, the global model $w^T$ generated by the \textit{FedInit} satisfies:
    \begin{equation}
        \frac{1}{T}\sum_{t=0}^{T-1}\mathbb{E}\Vert\nabla f(w^t)\Vert^2\leq \frac{D}{\lambda\eta KT} + \eta\frac{\kappa L\sigma}{\lambda N} - \frac{13\beta^2\kappa_\beta L^2}{\lambda\eta NKT}\Delta^T,
    \end{equation}
    where $D=f(w^0)-f(w^\star)$ is the initialized bias, $\sigma=\sigma_l^2+6K\sigma_g^2$ is the combination of biases of stochastic gradients and heterogeneity, $\kappa_\beta=1/(1-141\beta^2)$ is a constant related to $\beta$, $\kappa=8+78\beta^2\kappa_\beta^2$ is a constant related to $\beta$ and $\kappa_\beta$, $\lambda$ is a constant within $(0,\frac{1}{2})$, $\Delta^T=\frac{1}{C}\sum_{i\in\mathcal{C}}\mathbb{E}\Vert w_{i,K}^{T-1}-w^T\Vert^2$ is the consistency term at round $T$.
\end{theorem}
\begin{remark}
    Theorem~\ref{thm1} shows the general convergence bound of the \textit{FedInit} method. When $\beta=0$, it degrades to the vanilla \textit{FedAvg} method. The same as \textit{FedAvg}, it is subject to the initialization bias $D$ and the inherent variance $\sigma$. Differently, due to adopting RI, the consistency item could help to reduce the upper bound of the convergence rate. This effect from $\Delta^T$ term remains the same magnitude as the effect of $D$.
\end{remark}
\begin{remark}
    Let the learning rate $\eta$ be properly selected as $\eta=\mathcal{O}\left((N/KT)^{\frac{1}{2}}\right)$, we can bound the convergence rate as $\mathcal{O}\left((NKT)^{-\frac{1}{2}})\right)$, which has been proven as the optimal rate in the stochastic methods in FL under the general assumptions. The improvement of consistency term becomes weaker than the convergence rate by $\mathcal{O}(N^{-1})$. Although it could not further improve the rate, it is still catalytic to reduce the theoretical upper bound to some extend.
\end{remark}

\begin{theorem}
\label{thm2}
    Under Assumption~\ref{assumption:smooth}, \ref{assumption:PL}, \ref{assumption:stochastic}, and \ref{assumption:heterogeneity}, let participation ratio be $N/C$ where $1<N<C$, let the learning rate satisfies $\eta\leq 1/NKL$ where $K > 1$, and let the RI coefficient be a small positive constant, after total $T$ communication rounds, the global model $w^T$ generated by the \textit{FedInit} satisfies:
    \begin{equation}
        \mathbb{E}\left[f(w^T)-f(w^\star)\right]\leq e^{-\lambda\mu\eta KT}D + \eta\frac{\kappa L\sigma}{\lambda\mu N} + \mathcal{O}\left(\frac{\eta}{T}+\eta^2\right),
    \end{equation}
    where the coefficients are defined in Theorem~\ref{thm1}.
\end{theorem}
\begin{remark}
    Theorem~\ref{thm2} shows the general convergence bound of \textit{FedInit} method under the P\L-condition. The impact of the consistency is much weaker than the dominant term of $\sigma$, which can be considered that the primary consistency term does not burden the convergence bound. If we properly selected the learning rate as $\eta=\mathcal{O}\left(\log(NKT)/\lambda\mu KT\right)$, the convergence rate achieves at most $\widetilde{\mathcal{O}}\left((NKT)^{-1}\right)$, which still maintains the linear speedup property with $N$ and $K$.
\end{remark}

\begin{theorem}
\label{thm3}
    Under Assumption~\ref{assumption:smooth}, \ref{assumption:interpolation}, and \ref{assumption:global_interpolation}, let participation ratio be $N/C$ where $1<N<C$, let the learning rate satisfies $\eta\leq 1/aKL$ where $K > 1$, and let the RI coefficient be a small positive constant, after total $T$ communication rounds, the global model $w^T$ generated by the \textit{FedInit} satisfies:
    \begin{equation}
        \frac{1}{T}\sum_{t=0}^{T-1}\mathbb{E}\Vert\nabla f(w^t)\Vert^2\leq\frac{D}{\zeta\eta KT} - \frac{2\gamma_\beta\beta^2L}{\zeta\eta KT}\Delta^T,
    \end{equation}
    where $\gamma_\beta=1/(1-39\beta^2)$ is a constant related to $\beta$, $\zeta$ is a constant within $(0,\frac{1}{2})$, and other coefficients are defined in Theorem~\ref{thm1}.
\end{theorem}
\begin{remark}
    Theorem~\ref{thm3} shows the general convergence bound of \textit{FedInit} method under the interpolation conditions. The deep models are considered to maintain the high ability to handle each single data sample. Under this assumption, the negative impact of the variance will be diminished to very small when the optimization process converges. Its dominant term mainly comes from the initialization bias $D$. The same, RI still contributes a positive effort to reduce the convergence bound with the consistency term on the non-convex objective.
\end{remark}
\begin{remark}
    Let the learning rate $\eta$ be properly selected as $\eta=\mathcal{O}\left(N^{-1}\right)$, we can still bound the convergence rate as $\mathcal{O}\left((NKT)^{-1}\right)$. The improvement of consistency term maintains the same order with the initialization bias $D$. Improved optimization errors with RI can be felt more intuitively here, which benefits from a positive $\beta$.
\end{remark}

\begin{theorem}
\label{thm4}
    Under Assumption~\ref{assumption:smooth}, \ref{assumption:PL}, \ref{assumption:interpolation}, and \ref{assumption:global_interpolation}, let participation ratio be $N/C$ where $1<N<C$, let the learning rate satisfies $\eta\leq 1/aKL$ where $K > 1$, and let the RI coefficient be a small positive constant, after total $T$ communication rounds, the global model $w^T$ generated by the \textit{FedInit} satisfies:
    \begin{equation}
        \mathbb{E}\left[f(w^T)-f(w^\star)\right]\leq e^{-\zeta\mu\eta KT}D + R_\beta\eta^2K^2LD,
    \end{equation}
    where $R_\beta=228\mu\gamma_\beta\beta^2a^2b^2$ is a constant related to $\beta$, and other coefficients are defined in Theorem~\ref{thm1} and \ref{thm3}.
\end{theorem}
\begin{remark}
    Theorem~\ref{thm4} shows the general convergence bound of the \textit{FedInit} method under both P\L-condition and interpolation conditions. Let learning rate $\eta$ be properly selected as $\eta=\mathcal{O}\left(2\log(KT)/\zeta\mu KT\right)$, the convergence rate achieves at most $\widetilde{\mathcal{O}}\left(T^{-2}\right)$, which is dominated by the initialization bias $D$. It is consistent with the upper bound in Theorem~\ref{thm2} when stochastic variances are ignored under interpolations. We summarize the main conclusions of optimization in Table~\ref{table:optimization}.
\end{remark}

\subsection{Generalization Analysis}
In this part, we mainly illustrate the generalization analysis of our \textit{FedInit} method under two gradients properties~(bounded variance and interpolation). We adopt the uniform stability analysis which is widely adopted in previous literatures. We first introduce the definition of our analysis as follows, and then introduce the theoretical analysis and demonstrate the improvements of the RI technique in FL. All detailed proofs could be referred to the Appendix~\ref{app:gen}.\\

\noindent In the FL framework, we suppose there are $C$ clients participating in the training process as a set $\mathcal{C}=\{i\}_{i=1}^C$. Each client has a local dataset $\mathcal{S}_i=\{z_j\}_{j=1}^S$ with total $S$ data sampled from a specific unknown distribution $\mathcal{D}_i$. Now we define a re-sampled dataset $\widetilde{\mathcal{S}}_i$ which only differs from the dataset $\mathcal{S}_i$ on the $j^\star$-th data. We replace the $\mathcal{S}_{i^\star}$ with $\widetilde{\mathcal{S}}_{i^\star}$ and keep other $C-1$ local dataset, which composes a new set $\widetilde{\mathcal{C}}$. From the perspective of total data, $\mathcal{C}$ only differs from the $\widetilde{\mathcal{C}}$ at $j^\star$-th data on the $i^\star$-th client. Then, based on these two sets, our method could generate two output models, $w^T$ and $\widetilde{w}^T$ respectively, after $T$ communication rounds. By bounding the difference according to these two models, we can learn the stability and generalization efficiency. Our analysis mainly focuses comparisons of vanilla SGD, \textit{FedAvg}, and our proposed \textit{FedInit} method.

\begin{definition}[Uniform Stability~\citep{hardt2016train}]
    For these two models $w^T$ and $\widetilde{w}^T$ generated as introduced above, a general method satisfies $\epsilon$-uniformly stability if:
    \begin{equation}
          \sup_{z_j\sim\{\mathcal{D}_i\}}\mathbb{E}[f(w^T;z_j) - f(\widetilde{w}^T;z_j)]\leq \epsilon.
    \end{equation}
    Moreover, if a general method satisfies $\epsilon$-uniformly stability, then its generalization error could also be bounded as $\mathcal{E}_G\leq \sup_{z_j\sim\{\mathcal{D}_i\}}\mathbb{E}[f(w^T;z_j) - f(\widetilde{w}^T;z_j)]\leq \epsilon$~\citep{hardt2016train,zhang2022stability}.
\end{definition}

\begin{theorem}
\label{thm5}
    Under Assumption~\ref{assumption:smooth}, \ref{assumption:stochastic}, and \ref{assumption:assumption_Lipschitz}, let all conditions in the optimization process be satisfied, let the learning rate be selected as $\eta=\mathcal{O}\left(1/t\right)=c/t$ where $c$ is a constant, let $t_0$ be a specific round to firstly select the different data sample, and let $U=\sup f(w,z)$ be the upper bound, for arbitrary data sample $z$ followed the joint distribution $\{\mathcal{D}_i\}$, we have:
    \begin{align}
        \varepsilon_G 
        &\leq \mathbb{E}\Vert f(w^{T+1};z) - f(\widetilde{w}^{T+1};z)\Vert \nonumber \leq \frac{NUKt_0}{CS} + \frac{2\sigma_lL_G}{(1+2\beta)CSL}\left(\frac{T}{t_0}\right)^{cKL}.
    \end{align}
    Furthermore, to minimize the stability errors, we can select the proper observation point $t_0=\left[\frac{2\sigma_lL_G}{(1+2\beta)NUKL}\right]^{\frac{1}{1+cKL}}T^{\frac{cKL}{1+cKL}}$ and then we have:
    \begin{equation}
        \varepsilon_G\leq \frac{2}{CS}\left[\frac{2\sigma_lL_G}{(1+2\beta)L}\right]^{\frac{1}{1+cKL}}\left(NUKT\right)^{\frac{cKL}{1+cKL}}.
    \end{equation}
\end{theorem}
\begin{remark}
    Theorem~\ref{thm5} demonstrates the general generalization error bound of the proposed \textit{FedInit} method. When $\beta=0$, it degrades to the vanilla \textit{FedAvg} method. From the generalization bound, we mainly focuses on the terms of total number of data samples and the training length $T$ and $K$. If we adopt the vanilla SGD to train a single model with total $CS$ data samples after $T$ iterations, \cite{hardt2016train} have provided the general stability on the smooth and non-convex objectives as $\mathcal{O}\left(\frac{T^{\frac{c L}{1+c L}}}{CS}\right)$. In FL, our analysis indicates that it achieves $\mathcal{O}\left(\frac{(NKT)^{\frac{cKL}{1+cKL}}}{CS}\right)$. We first summarize the comparison between vanilla SGD and \textit{FedAvg}~($\beta=0$) in Table~\ref{generalization1}.
    \begin{table}[h]
    \caption{Comparison between vanilla SGD and \textit{FedAvg}.}
    \vspace{0.1cm}
    \label{generalization1}
    \centering
    \begin{tabular}{c|cc}
        \toprule
          &  SGD & \textit{FedAvg} \\
        \midrule
        number of samples & $\mathcal{O}\left(\frac{1}{CS}\right)$ &  $\mathcal{O}\left(\frac{N^{\frac{cKL}{1+cKL}}}{CS}\right)$\\
        number of iterations & $\mathcal{O}\left(T^\frac{cL}{1+cL}\right)$ &  $\mathcal{O}\left((TK)^\frac{cKL}{1+cKL}\right)$ \\
        \bottomrule
    \end{tabular}
    \vspace{0.1cm}
\end{table}

\noindent From above table, we clearly see the damage to the stability due to the local training process on local private dataset. When $K=1$ and $N=1$, $\textit{FedAvg}$ degrades to the vanilla SGD method. When $K=1$ and $N>1$, $\textit{FedAvg}$ degrades to the mini-batch SGD with the batchsize of $N$. Therefore, we know that to achieve the same stability error, FL always requires more data samples. This is the inherent bias in FL. The worst case in FL is to let $N=C$ as the full participation and it achieves the $\mathcal{O}\left(\left(\frac{1}{C}\right)^{\frac{cKL}{1+cKL}}\frac{1}{S}\right)$. A large number of local clients will seriously hinder the stability of the global model. Similarly, for the high stability, the number of participating clients per round also needs to be limited.
\end{remark}
\begin{remark}
    Obviously, when we select a small positive $\beta$, it reduces the generalization error by $\left(\frac{1}{1+2\beta}\right)^{\frac{1}{1+cKL}}$, which illustrates the advantages of RI. It effectively improves the stability of the global model and has enormous utility in practice. As $K$ increases, the performance improvement brought by $\beta$ will gradually become smaller. This conclusion is also intuitive. RI moves away from the last local state as a new initialization, but when $K$ is large enough, the advantages of this compensation will be greatly reduced. The local training will make the local model overfit to the local dataset. When $K\rightarrow \infty$, the impact of $\left(\frac{1}{1+2\beta}\right)^{\frac{1}{1+cKL}}\rightarrow 1$.
\end{remark}

\begin{theorem}
\label{thm6}
    Under Assumption~\ref{assumption:smooth}, \ref{assumption:interpolation}, and \ref{assumption:assumption_Lipschitz}, let all conditions in the optimization process be satisfied, let the learning rate be selected as $\eta=\mathcal{O}\left(1/t\right)=c/t$ where $c$ is a constant, let $t_0$ be a specific round to firstly select the different data sample, and let $U=\sup f(w,z)$ be the upper bound, for arbitrary data sample $z$ followed the joint distribution $\{\mathcal{D}_i\}$, furthermore, let $\beta$ be a decayed sequence $\{\beta^t\}_{t=0}^T$ by the round $t$, we have:
    \begin{equation}
        \varepsilon_G
        \leq \mathbb{E}\Vert f(w^{T+1};z) - f(\widetilde{w}^{T+1};z)\Vert\leq e^{2\sum_{t=1}^{T}\beta^t}\frac{abcL_G^2K}{2\beta^0 CS}.
    \end{equation}
\end{theorem}
\begin{remark}
    Theorem~\ref{thm6} demonstrates the general generalization error bound of the proposed \textit{FedInit} method under the interpolation condition. It achieves the $\mathcal{O}\left(\frac{K}{CS}\right)$ rate under the proper selection of $\beta$. Here the $\beta$ must be decayed by the round $t$ and let $\sum_{t=1}^T\beta^t$ be at least a constant bound. Actually, the coefficient $\beta$ includes the decayed rate of the learning rate which could be considered as a quasi-learning rate term. Increasing the local interval $K$ still damages the stability and draws a negative impact on the final convergence. This error asymptotically approximates the error of the vanilla SGD~\citep{hardt2016train}.
\end{remark}

\subsection{Divergence Term $\Delta^t$}
In the former two parts, we provide the complete theorem to understand the optimization error $\mathcal{E}_O$ and generalization error $\mathcal{E}_G$. In this part, we focus on the analysis of the divergence term of our proposed \textit{FedInit} method. Due to the RI at the beginning of each communication round, according to the Algorithm~\ref{algorithm}, we have $w_{i,K}^t=w^t+\beta(w^t - w_{i,K}^{t-1})-\eta\sum_{k=0}^{K-1}g_{i,k}^t$. Thus, we have the following recursive relationship:
\begin{align}
\label{tx:divergence_relationship}
    \underbrace{w^{t+1} - w_{i,K}^{t}}_{\textit{local divergence at $t+1$}} = \beta\underbrace{(w_{i,K}^{t-1} - w^t)}_{\textit{local divergence at $t$}}  + \underbrace{(w^{t+1} -w^t)}_{\textit{global update}} + \underbrace{\sum_{k=0}^{K-1}\eta g_{i,k}^t}_{\textit{local updates}}.
\end{align}
According to the recursive formulation~(\ref{tx:divergence_relationship}), we can bound the divergence $\Delta^t$. Detailed proofs are stated in Appendix~\ref{app:opt}.
\begin{theorem}
\label{thm7}
    Under Assumption~\ref{assumption:smooth}, \ref{assumption:stochastic}, and \ref{assumption:heterogeneity}, let all conditions in Theorem~\ref{thm1} and \ref{thm2} hold, we can bound the divergence term:
    \begin{equation}
        \frac{1}{T}\sum_{t=0}^{T-1}\Delta^t\leq \eta\frac{804C_\beta K}{\lambda T}D + \eta^2J_\beta K\sigma = \mathcal{O}\left(\frac{\eta}{T}+\eta^2\right),
    \end{equation}
    where $C_\beta=\kappa_\beta/(1-48\beta^2\kappa_\beta)$ is a constant related to $\kappa_\beta$ and $\beta$, and $J_\beta=C_\beta(132+804\kappa/\lambda N)$ is a constant related to $C_\beta$. Other constant coefficients are defined in Theorem~\ref{thm1}. 
\end{theorem}

\begin{remark}
    Theorem~\ref{thm7} corresponds to the general analysis of Theorem~\ref{thm1} and \ref{thm2}. When the learning rate is selected as $\mathcal{O}\left((N/KT)^{\frac{1}{2}}\right)$~(Theorem~\ref{thm1}), the dominant term comes from the stochastic variance and heterogeneity and achieves the $\mathcal{O}\left(T^{-1}\right)$ rate. When the learning rate is selected as $\mathcal{O}\left(\log(NKT)/\lambda\mu KT\right)$~(Theorem~\ref{thm2}), the dominant term comes from the initial bias $D$ and achieves the $\widetilde{\mathcal{O}}\left(T^{-2}\right)$ rate. It also indicates that if the learning rate is selected as a constant, the consistency will always maintain a constant upper bound which is dominated by the stochastic variance.
\end{remark}

\begin{theorem}
\label{thm8}
    Under Assumption~\ref{assumption:smooth}, \ref{assumption:interpolation}, and \ref{assumption:global_interpolation}, let all conditions in Theorem~\ref{thm3} and \ref{thm4} hold, we can bound the divergence term:
    \begin{equation}
        \frac{1}{T}\sum_{t=0}^{T-1}\Delta^t \leq \eta\frac{R_\beta K}{2\mu\beta^2\zeta T}D = \mathcal{O}\left(\frac{\eta}{T}\right),
    \end{equation}
    where all coefficients are defined in Theorem~\ref{thm4}.
\end{theorem}
\begin{remark}
    Theorem~\ref{thm8} demonstrates the divergence term will always diminish by at least $\mathcal{O}(T^{-1})$ rate under interpolation condition. The conclusion under the strong assumption reveals the potential prospects of FL on large models. When the model could handle each separate data sample, the consensus of the local clients during the training process can always be effectively guaranteed. Meanwhile, it also illustrates that improving the capabilities of the model can greatly compensate for the inherent bias in the training model in FL, which can reduce the negative impact of ``client drift".
\end{remark}

\noindent\textbf{Discussion.} We provide a comprehensive analysis on both optimization and generalization of our proposed \textit{FedInit} method to further understand how and why it can improve efficiency in FL. From the Theorem~\ref{thm1} and \ref{thm3}, we can know that under the general assumptions, compared with the \textit{FedAvg}, RI contributes to reducing the convergence bound by an additional negative term of divergence. This negative term is independent of the methods adopted, which only comes from personalized relaxed initialization. In other words, if we incorporate RI into advanced benchmarks, e.g. for \textit{SCAFFOLD}, it also could be proven that there is a similar term to reduce its convergence bound. Furthermore, we provide how it helps to reduce the generalization error bound in Theorem~\ref{thm5} under the general assumptions. By adopting RI, the stability is enhanced with a $\left(\frac{1}{1+2\beta}\right)^{\frac{1}{1+cKL}}$ rate. When the local interval $K$ increases, appropriately increasing $\beta$ can also maintain high efficiency and model stability.

\section{Experiments}
\label{tx:experiments}
In this section, we mainly introduce our experimental studies. We first introduce the experimental setups, including the benchmarks, general settings, and hyperparameter selections. Then we show our experiments and provide the corresponding understanding according to the theoretical analysis, specifically for the efficiency of relaxed initialization.

\subsection{Experimental Setups}
\noindent
\textbf{Benchmarks.} Benchmarks selected in our paper: \textit{FedAvg}~\citep{mcmahan2017communication} proposes a general FL paradigm; \textit{FedAdam}~\citep{reddi2020adaptive} studies the efficiency of the adaptive optimizer in FL; \textit{SCAFFOLD}~\citep{karimireddy2020scaffold}, \textit{FedDyn}~\citep{acar2021federated}, and \textit{FedCM}~\citep{xu2021fedcm} learn the ``client-drift" problem and adopt the variance reduction technique, ADMM, and the client-level momentum respectively in FL to alleviate its negative impact; \textit{FedSAM}~\citep{qu2022generalized} uses the local SAM objective instead of the vanilla empirical risk objective to search for a smooth loss landscape.\\

\noindent
\textbf{Models and Dataset.} Here we briefly introduce the setups in our experiments. We validate the proposed \textit{FedInit} method on the classical CIFAR-10$/$100 dataset~\citep{cifar100}. To generate local heterogeneity, we follow \cite{dirichlet} to split the local clients through the Dirichlet sampling via the coefficient $D_r$ to control the heterogeneous level and follow \cite{sun2023fedspeed} to adopt the sampling with replacement to enhance the heterogeneity level. We test on the ResNet-18-GN~\citep{resnet,gn_bn} and VGG-11~\citep{simonyan2014very} to validate its efficiency. Actually, when the heterogeneity is strong, the performance of personalized initialization will be better. To better demonstrate the performance of our proposed method, we add additional noises to the dataset. Specifically, we first introduce the \textit{client-based biases}. Among clients, we assume that the data samples are obtained differently. Because the local dataset is private and its construction is unknown, i.e., they are collected from different machines or cameras. Therefore, we change the strength of the $RGB$ channels with a random Gaussian noise for different clients. The second noise is the \textit{category-based biases}. We assume that samples for each category also contain heterogeneity. In our experiments, we add different brightness perturbations to the samples in each category by a random Gaussian noise. Based on these two noises, local heterogeneity is significantly enlarged.\\

\noindent \textbf{Hyperparameters.} For each benchmark in our experiments,
we adopt two coefficients $D_r=0.1$ and $0.6$ for each dataset to generate different heterogeneity. We generally select the local learning rate $\eta=0.1$ and global learning rate $\eta=1$ on all setups except for \textit{FedAdam} we use $0.1$. The learning rate decay is set as multiplying $0.998$ per round except for \textit{FedDyn} we use $0.9995$. We train 500 rounds on CIFAR-10 and 800 rounds on CIFAR-100 to achieve the stable test accuracy. The participation ratios are selected as $10\%$ and $5\%$ respectively of total $100$ and $200$ clients. To eliminate large variances due to training instability, all results are first smoothed by the convolution filter and then we select the maximization value in the last $50$ communication rounds.

\subsection{Main Results}
\begin{table*}[t]
\centering
\renewcommand{\arraystretch}{1}
\caption{Test accuracy on the CIFAR-10 dataset. We test two active ratios on each dataset. On each setup, we test two Dirichlet splittings, and each result tests 3 times. The table reports results on ResNet-18-GN~(upper) and VGG-11~(lower).}
\label{tx:tb_results_10}
\vspace{-0.1cm}
\setlength{\tabcolsep}{2.5mm}{\begin{tabular}{@{}lcc|cc@{}}
\toprule
\multirow{3}{*}[-1.5ex]{\centering Method} & \multicolumn{4}{c}{CIFAR-10} \\ 
\cmidrule(lr){2-5}
\multicolumn{1}{c}{} & \multicolumn{2}{c}{$10\%$-100 clients} & \multicolumn{2}{c}{$5\%$-200 clients} \\ 
\cmidrule(lr){2-3} \cmidrule(lr){4-5} 
\multicolumn{1}{c}{} & \multicolumn{1}{c}{Dir-0.6}  & Dir-0.1 & \multicolumn{1}{c}{Dir-0.6}  & Dir-0.1 \\ 
\cmidrule(lr){1-5}               
FedAvg              & \multicolumn{1}{c}{$78.77_{\pm.11}$} & $72.53_{\pm.17}$ & \multicolumn{1}{c}{$74.81_{\pm.18}$} & $70.65_{\pm.21}$ \\ 
FedAdam             & \multicolumn{1}{c}{$76.52_{\pm.14}$} & $70.44_{\pm.22}$ & \multicolumn{1}{c}{$73.28_{\pm.18}$} & $68.87_{\pm.26}$ \\
FedSAM              & \multicolumn{1}{c}{$79.23_{\pm.22}$} & $72.89_{\pm.23}$ & \multicolumn{1}{c}{$75.45_{\pm.19}$} & $71.23_{\pm.26}$ \\ 
SCAFFOLD            & \multicolumn{1}{c}{$81.37_{\pm.17}$} & $75.06_{\pm.16}$ & \multicolumn{1}{c}{$78.17_{\pm.28}$} & $74.24_{\pm.22}$ \\ 
FedDyn              & \multicolumn{1}{c}{$82.43_{\pm.16}$} & $75.08_{\pm.19}$ & \multicolumn{1}{c}{$79.96_{\pm.13}$} & $74.15_{\pm.34}$ \\ 
FedCM               & \multicolumn{1}{c}{$81.67_{\pm.17}$} & $73.93_{\pm.26}$ & \multicolumn{1}{c}{$79.49_{\pm.17}$} & $73.12_{\pm.18}$ \\ 
\textbf{FedInit}    & \multicolumn{1}{c}{$\textbf{83.11}_{\pm.29}$} & $\textbf{75.95}_{\pm.19}$ & \multicolumn{1}{c}{$\textbf{80.58}_{\pm.20}$} & $\textbf{74.92}_{\pm.17}$ \\
\cmidrule(lr){1-5}               
FedAvg              & \multicolumn{1}{c}{$85.28_{\pm.12}$} & $78.02_{\pm.22}$ & \multicolumn{1}{c}{$81.23_{\pm.14}$} & $74.89_{\pm.25}$ \\ 
FedAdam             & \multicolumn{1}{c}{$86.44_{\pm.13}$} & $77.55_{\pm.28}$ & \multicolumn{1}{c}{$81.05_{\pm.23}$} & $74.04_{\pm.17}$ \\ 
FedSAM              & \multicolumn{1}{c}{$86.37_{\pm.22}$} & $79.10_{\pm.07}$ & \multicolumn{1}{c}{$81.76_{\pm.26}$} & $75.22_{\pm.13}$ \\ 
SCAFFOLD            & \multicolumn{1}{c}{$87.73_{\pm.17}$} & $81.98_{\pm.19}$ & \multicolumn{1}{c}{$84.81_{\pm.15}$} & $79.04_{\pm.16}$ \\ 
FedDyn              & \multicolumn{1}{c}{$87.35_{\pm.19}$} & $82.70_{\pm.24}$ & \multicolumn{1}{c}{$84.84_{\pm.19}$} & $\textbf{80.01}_{\pm.22}$\\ 
FedCM               & \multicolumn{1}{c}{$86.80_{\pm.33}$} & $79.85_{\pm.29}$ & \multicolumn{1}{c}{$83.23_{\pm.31}$} & $76.42_{\pm.36}$ \\ 
\textbf{FedInit}    & \multicolumn{1}{c}{$\textbf{88.47}_{\pm.22}$} & $\textbf{83.51}_{\pm.13}$ & \multicolumn{1}{c}{$\textbf{85.36}_{\pm.19}$} & $79.73_{\pm.14}$ \\
\bottomrule
\end{tabular}}
\end{table*}

\begin{table*}[t]
\centering
\renewcommand{\arraystretch}{1}
\caption{Test accuracy on the CIFAR-10$/$100 dataset. We test two active ratios on each dataset. On each setup, we test two Dirichlet splittings, and each result tests 3 times. The table reports results on ResNet-18-GN~(upper) and VGG-11~(lower).}
\label{tx:tb_results_100}
\vspace{-0.1cm}
\setlength{\tabcolsep}{2.5mm}{\begin{tabular}{@{}lcc|cc@{}}
\toprule
\multirow{3}{*}[-1.5ex]{\centering Method} & \multicolumn{4}{c}{CIFAR-100} \\ 
\cmidrule(lr){2-5}
\multicolumn{1}{c}{} & \multicolumn{2}{c}{$10\%$-100 clients} & \multicolumn{2}{c}{$5\%$-200 clients}\\ 
\cmidrule(lr){2-3} \cmidrule(lr){4-5} 
\multicolumn{1}{c}{} & \multicolumn{1}{c}{Dir-0.6}  & Dir-0.1 & \multicolumn{1}{c}{Dir-0.6}  & Dir-0.1 \\ 
\cmidrule(lr){1-5}               
FedAvg              & \multicolumn{1}{c}{$46.35_{\pm.15}$} & $42.62_{\pm.22}$ & \multicolumn{1}{c}{$44.70_{\pm.22}$} & $40.41_{\pm.33}$ \\ 
FedAdam             & \multicolumn{1}{c}{$48.35_{\pm.17}$} & $40.77_{\pm.31}$ & \multicolumn{1}{c}{$44.33_{\pm.26}$} & $38.04_{\pm.25}$ \\
FedSAM              & \multicolumn{1}{c}{$47.51_{\pm.26}$} & $43.43_{\pm.12}$ & \multicolumn{1}{c}{$45.98_{\pm.27}$} & $40.22_{\pm.27}$ \\ 
SCAFFOLD            & \multicolumn{1}{c}{$51.98_{\pm.23}$} & $\textbf{44.41}_{\pm.15}$ & \multicolumn{1}{c}{$50.70_{\pm.29}$} & $41.83_{\pm.29}$ \\ 
FedDyn              & \multicolumn{1}{c}{$50.82_{\pm.19}$} & $42.50_{\pm.28}$ & \multicolumn{1}{c}{$47.32_{\pm.21}$} & $41.74_{\pm.21}$ \\ 
FedCM               & \multicolumn{1}{c}{$51.56_{\pm.20}$} & $43.03_{\pm.26}$ & \multicolumn{1}{c}{$50.93_{\pm.19}$} & $42.33_{\pm.19}$ \\ 
\textbf{FedInit}    & \multicolumn{1}{c}{$\textbf{52.21}_{\pm.09}$} & $44.22_{\pm.21}$ & \multicolumn{1}{c}{$\textbf{51.16}_{\pm.18}$} & $\textbf{43.77}_{\pm.36}$ \\
\cmidrule(lr){1-5}               
FedAvg              & \multicolumn{1}{c}{$53.46_{\pm.25}$} & $50.53_{\pm.20}$ & \multicolumn{1}{c}{$47.55_{\pm.13}$} & $45.05_{\pm.33}$ \\ 
FedAdam             & \multicolumn{1}{c}{$55.56_{\pm.29}$} & $53.41_{\pm.18}$ & \multicolumn{1}{c}{$51.33_{\pm.25}$} & $47.26_{\pm.21}$ \\ 
FedSAM              & \multicolumn{1}{c}{$54.85_{\pm.31}$} & $51.88_{\pm.27}$ & \multicolumn{1}{c}{$48.65_{\pm.21}$} & $46.58_{\pm.28}$ \\ 
SCAFFOLD            & \multicolumn{1}{c}{$\textbf{59.45}_{\pm.17}$} & $56.67_{\pm.24}$ & \multicolumn{1}{c}{$53.73_{\pm.32}$} & $50.08_{\pm.19}$ \\ 
FedDyn              & \multicolumn{1}{c}{$56.13_{\pm.18}$} & $53.97_{\pm.11}$ & \multicolumn{1}{c}{$51.74_{\pm.18}$} & $48.16_{\pm.17}$ \\ 
FedCM               & \multicolumn{1}{c}{$53.88_{\pm.22}$} & $50.73_{\pm.35}$ & \multicolumn{1}{c}{$47.83_{\pm.19}$} & $46.33_{\pm.25}$ \\ 
\textbf{FedInit}    & \multicolumn{1}{c}{$58.84_{\pm.11}$} & $\textbf{57.22}_{\pm.21}$ & \multicolumn{1}{c}{$\textbf{54.12}_{\pm.08}$} & $\textbf{50.27}_{\pm.29}$     \\
\bottomrule
\end{tabular}}
\end{table*}

\noindent In Table~\ref{tx:tb_results_10} and \ref{tx:tb_results_100}, our \textit{FedInit} method performs better than the other benchmarks with good stability across different experimental setups. On the test of ResNet-18-GN model on CIFAR-10, it achieves about 3.42$\%$ improvement than the vanilla \textit{FedAvg} on the high heterogeneous splitting with $D_r=0.1$. When the participation ratio decreases to $5\%$, the accuracy drops only about $0.1\%$ while \textit{FedAvg} drops almost $1.88\%$. Similar results on CIFAR-100, when the participation ratio decreases, \textit{FedInit} can still achieve the test accuracy of $43.77\%$, while the second best method \textit{SCAFFOLD} drops $3.21\%$ from the former. This clearly demonstrates the good generalization of the proposed \textit{FedInit} on different participation ratios. Judging from the overall results, the improvement of adopting RI is significantly considerable.

In addition, in Table~\ref{tx:tb_plusin}, we incorporate the relaxed initialization~(RI) into the other benchmarks to further validate its benefit. ``~-~" means the vanilla benchmarks, and ``~+RI~" means adopting the relaxed initialization. It shows that the relaxed initialization holds the promising potential to further enhance the performance. Actually, \textit{FedInit} could be considered as (RI + \textit{FedAvg}), whose improvement achieves about over $3\%$ on each setup. Table~\ref{tx:tb_results_10} and \ref{tx:tb_results_100} shows the poor performance of the vanilla \textit{FedAvg}. Nevertheless, when adopting the RI, \textit{FedInit} remains above most benchmarks on several setups. When the RI is incorporated into other benchmarks, it helps them to achieve higher performance without additional communication costs. RI does not hinder other algorithms and significantly improves their performance. For instance, it helps the classical \textit{SCAFFOLD} method to improve about $2\%\sim3\%$, which makes the \textit{SCAFFOLD} method easily achieve the SOTA results. This is consistent with the theoretical analysis we mentioned above. The improvement in RI is due to the training mode of FL and is independent of the method adopted for local training. As a very lightweight computing plug-in, it can help many advanced algorithms further improve their performance, which has high compatibility and scalability.

\subsection{Sensitivity on $K$ and $\beta$}
\begin{figure}[h]
\centering
    \subfigure[\small Different $K$.]{
        \includegraphics[width=0.45\textwidth]{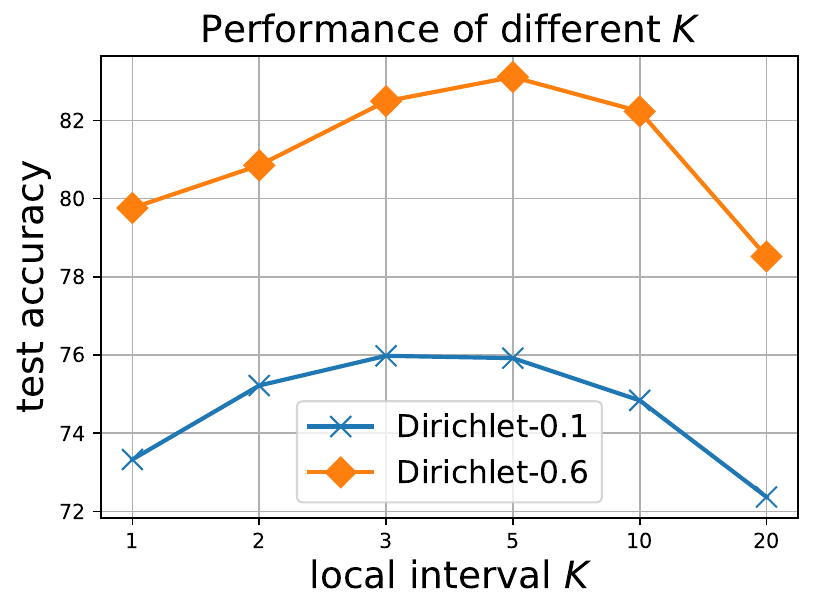}}\!\!\!\!
    \subfigure[\small Different $\beta$.]{
	\includegraphics[width=0.45\textwidth]{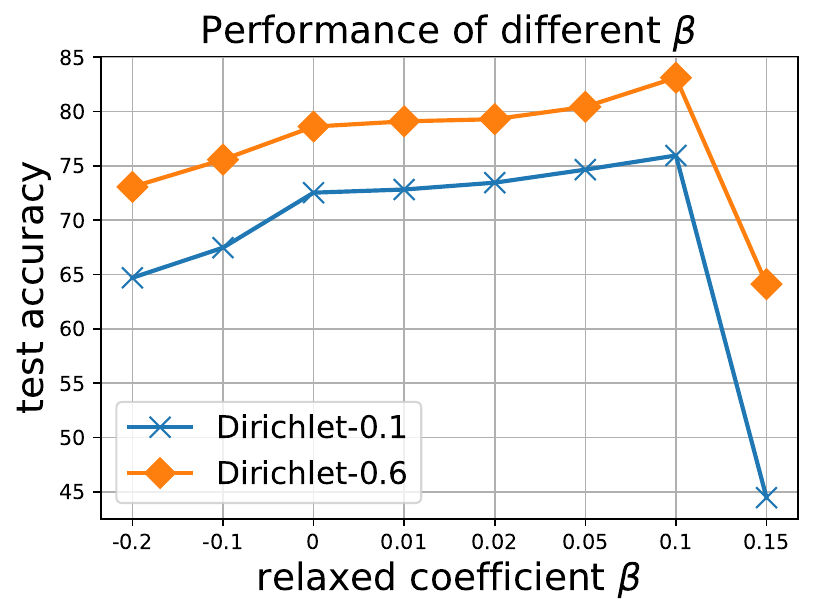}}
\caption{\small Sensitivity studies of local intervals $K$ and relaxed coefficient $\beta$ of the \textit{FedInit} method on CIFAR-10. To fairly compare their efficiency.}
\label{hyperparamters}
\end{figure}

\begin{figure*}[t]
\centering
    \subfigure[Dir-0.6 10\%-100, \textit{FedAvg}\ +\ RI~(\textit{FedInit}).]{
        \includegraphics[width=0.24\textwidth]{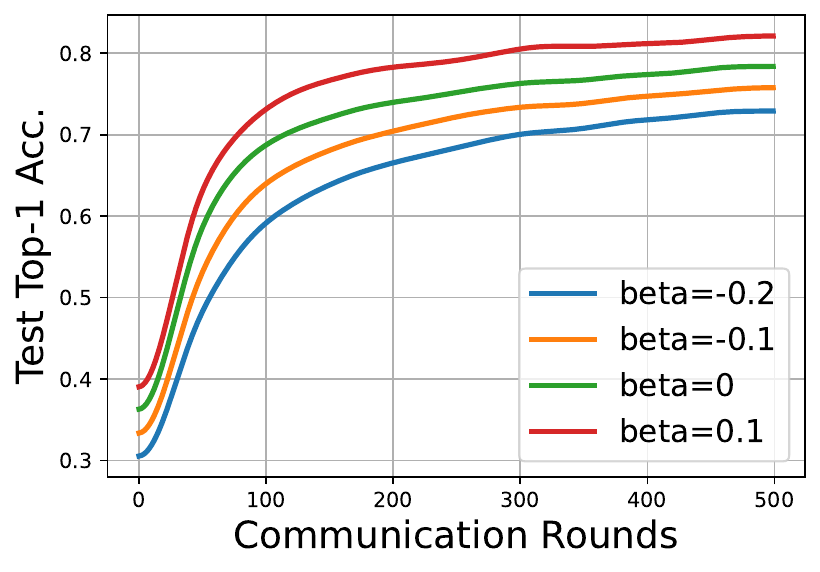}\!
	\includegraphics[width=0.24\textwidth]{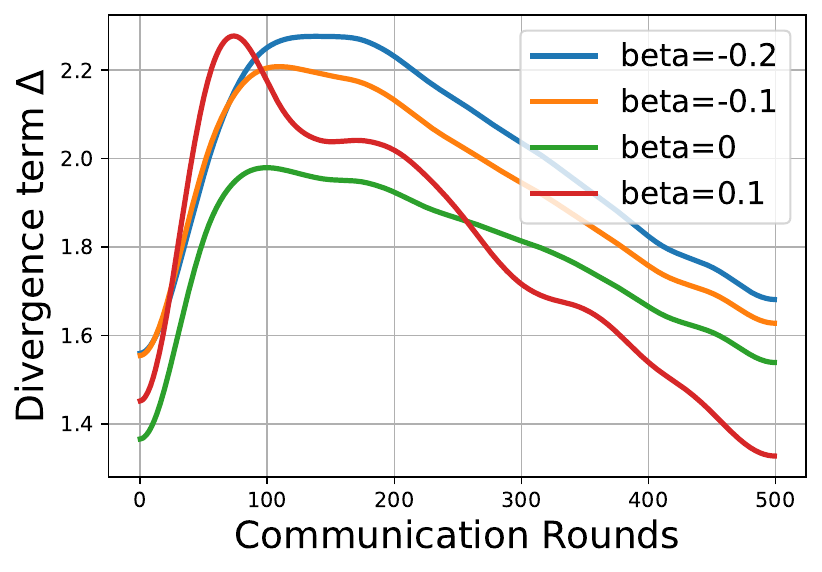}}\!
    \subfigure[Dir-0.1 5\%-200 with \textit{FedAvg}\ +\ RI~(\textit{FedInit}).]{
	\includegraphics[width=0.24\textwidth]{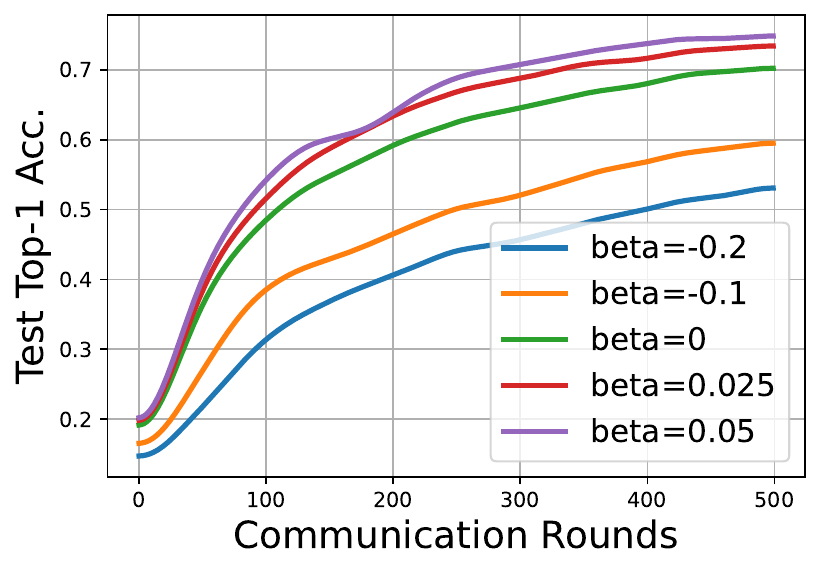}\!
	\includegraphics[width=0.24\textwidth]{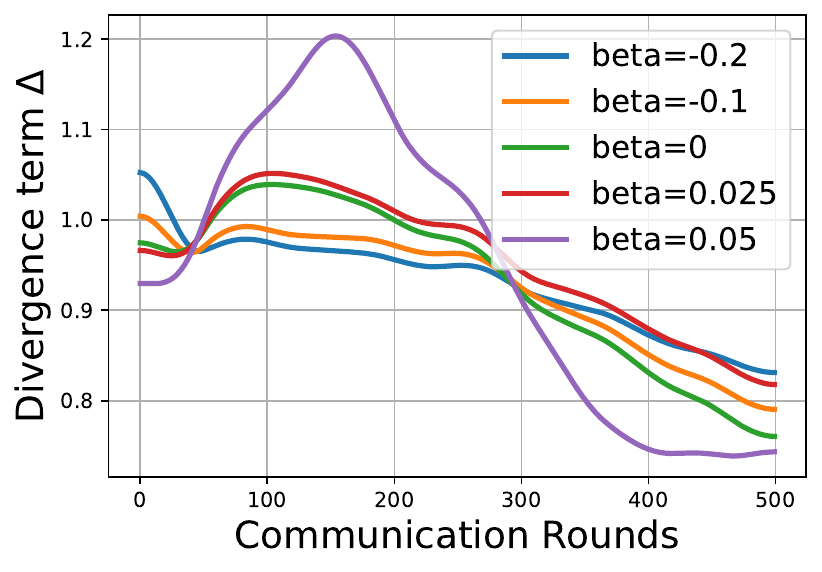}}
    \subfigure[Dir-0.6 10\%-100 with \textit{SCAFFOLD}\ +\ RI.]{
        \includegraphics[width=0.24\textwidth]{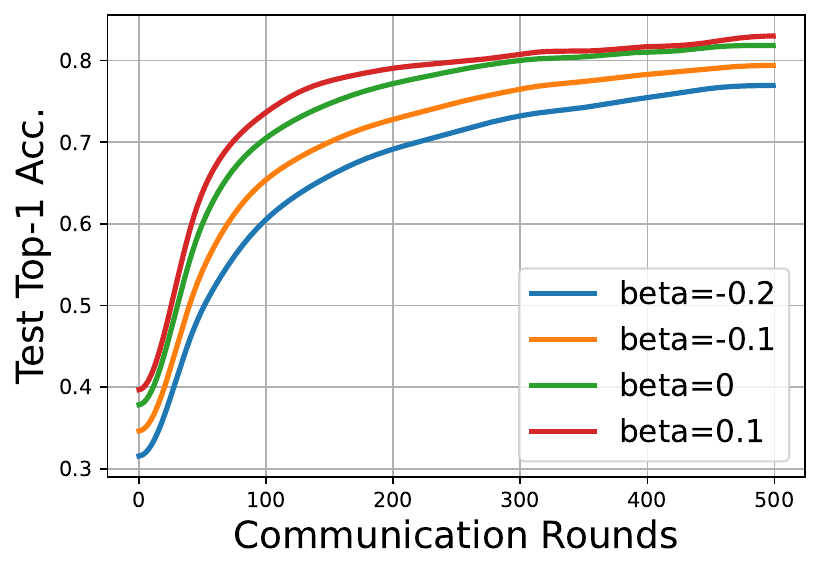}\!
	\includegraphics[width=0.24\textwidth]{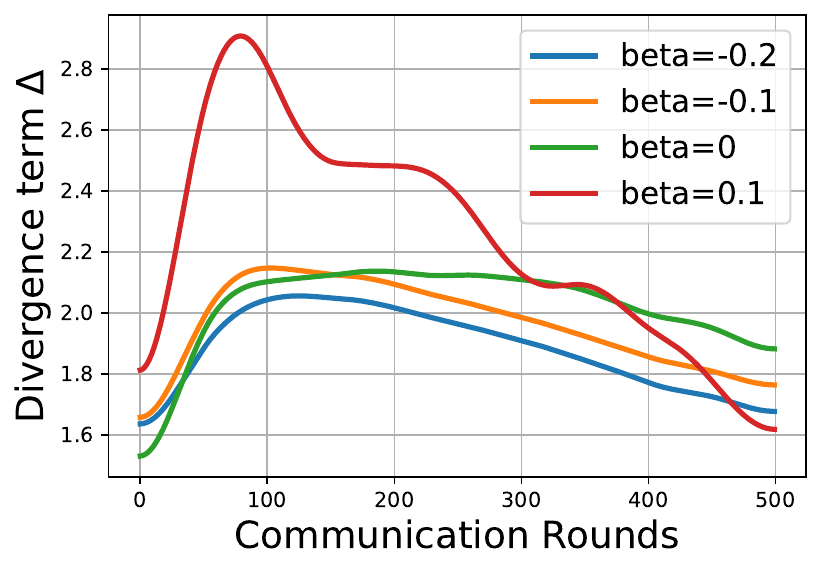}}\!
    \subfigure[Dir-0.1 5\%-200 with \textit{SCAFFOLD}\ +\ RI.]{
	\includegraphics[width=0.24\textwidth]{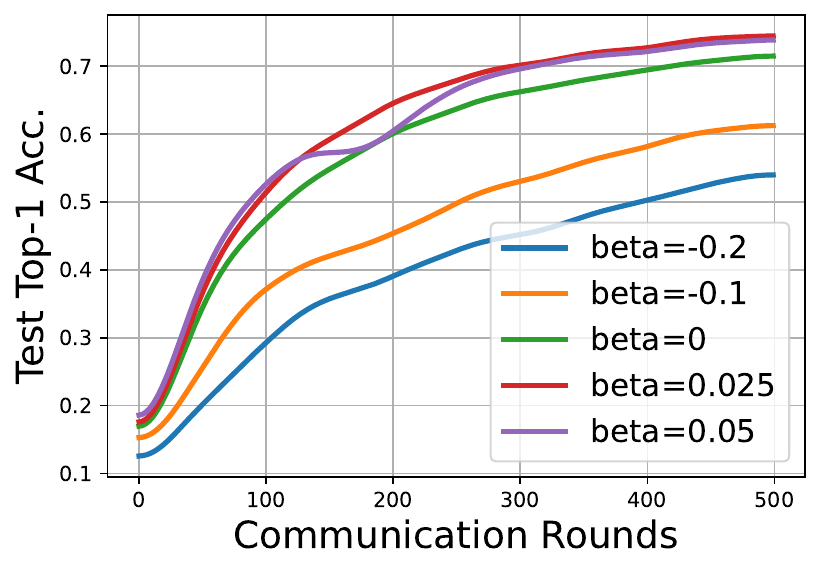}\!
	\includegraphics[width=0.24\textwidth]{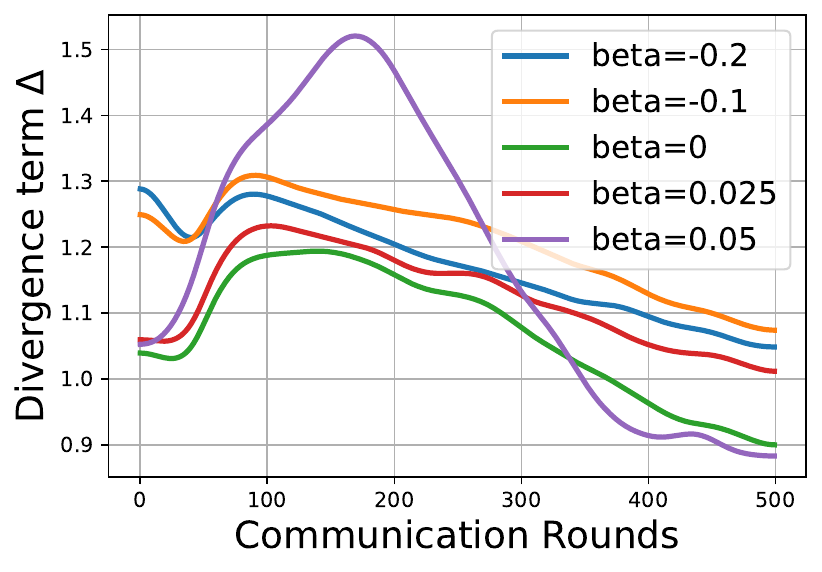}}
\caption{The accuracy and divergence of different $\beta$. (a) and (b) are tested with \textit{FedAvg}\ +\ RI~(\textit{FedInit}), while (c) and (d) for \textit{SCAFFOLD}\ +\ RI. (a) and (c) are tested on the Dir-0.6 10\%-100 setups, while (b) and (d) for Dir-0.1 5\%-200 setups.}
\label{divergence}
\end{figure*}
The excess risk and test error of \textit{FedInit} indicates there exist best selections for local interval $K$ and relaxed coefficient $\beta$, respectively. In this part, we test a series of selections to validate our conclusions. To be aligned with previous studies, we denote $K$ as \textit{training epochs} in the experiments instead of the iterations in the theoretical analysis. Their relationship is: 1 \textit{epoch} = $\left(S/\textit{batchsize}\right)$ iterations. In Fig.~\ref{hyperparamters}~(a), we clearly see that there is the best selection on the local interval $K$. As local interval $K$ increases, test accuracy rises first and then decreases. Our analysis provides a comprehensive explanation of this phenomenon. The optimization error decreases as $K$ increases when it is small. When $K$ exceeds the threshold, the divergence term in generalization cannot be ignored. Therefore, the test accuracy will be significantly affected. In Fig.~\ref{hyperparamters}~(b), we can see that the selection of the $\beta$ is stable within its valid range, which has great potential to improve performance. When it is larger than the threshold, the training process will diverge quickly. This phenomenon is highly consistent with our analysis. To simplify the analysis, we omit the full terms of the upper bound of $\beta$ and select a constant bound in our theorems. In fact, the complete upper bound of $\beta$ is related to the learning rate $\eta$ and the local interval $K$. However, in the experiments, we find that a simple constant selection can offer high improvements.

\begin{table*}[t]
    \centering
    \renewcommand{\arraystretch}{1}
    \caption{We incorporate the relaxed initialization~(RI) into the different benchmarks to test their improvements on the ResNet-18-GN model on the CIFAR-10 dataset under the same settings and hyperparameters selections.}
    \label{tx:tb_plusin}
    \begin{tabular}{@{}lcc|cc|cc|cc@{}}
    \toprule
    \multirow{3}{*}[-1.5ex]{\centering Method} & \multicolumn{4}{c}{$10\%$-100 clients} & \multicolumn{4}{c}{$5\%$-200 clients} \\ 
    \cmidrule(lr){2-5} \cmidrule(lr){6-9}
    \multicolumn{1}{c}{} & \multicolumn{2}{c}{Dir-0.6} & \multicolumn{2}{c}{Dir-0.1} & \multicolumn{2}{c}{Dir-0.6} & \multicolumn{2}{c}{Dir-0.1}\\ 
    \cmidrule(lr){2-3} \cmidrule(lr){4-5} \cmidrule(lr){6-7} \cmidrule(lr){8-9} 
    \multicolumn{1}{c}{} & \multicolumn{1}{c}{-}  & +RI & \multicolumn{1}{c}{-}  & +RI & \multicolumn{1}{c}{-}  & +RI & \multicolumn{1}{c}{-}  & +RI \\ 
    \cmidrule(lr){1-9}               
    FedAvg              & \multicolumn{1}{c}{78.77} & 83.11 & \multicolumn{1}{c}{72.53} & 75.95 & \multicolumn{1}{c}{74.81} & 80.58 & \multicolumn{1}{c}{70.65} & 74.92 \\ 
    FedAdam             & \multicolumn{1}{c}{76.52} & 78.33 & \multicolumn{1}{c}{70.44} & 72.55 & \multicolumn{1}{c}{73.28} & 78.33 & \multicolumn{1}{c}{68.87} & 71.34 \\
    FedSAM              & \multicolumn{1}{c}{79.23} & \textbf{83.36} & \multicolumn{1}{c}{72.89} & 76.34 & \multicolumn{1}{c}{75.45} & 80.66 & \multicolumn{1}{c}{71.23} & 75.08 \\ 
    SCAFFOLD            & \multicolumn{1}{c}{81.37} & 83.27 & \multicolumn{1}{c}{75.06} & \textbf{77.30} & \multicolumn{1}{c}{78.17} & \textbf{81.02} & \multicolumn{1}{c}{74.24} & \textbf{76.22} \\ 
    FedDyn              & \multicolumn{1}{c}{82.43} & 81.91 & \multicolumn{1}{c}{75.08} & 75.11 & \multicolumn{1}{c}{79.96} & 79.88 & \multicolumn{1}{c}{74.15} & 74.34 \\ 
    FedCM               & \multicolumn{1}{c}{81.67} & 81.77 & \multicolumn{1}{c}{73.93} & 73.71 & \multicolumn{1}{c}{79.49} & 79.72 & \multicolumn{1}{c}{73.12} & 72.98 \\ 
    \bottomrule
    \end{tabular}
\end{table*}

\subsection{Divergence Term $\Delta^t$}
In this part, we mainly test the consistency level of different $\beta$. The coefficient $\beta$ controls the divergence level of the local initialization states. We select the \textit{FedAvg} and \textit{SCAFFOLD} to show the efficiency of the proposed relaxed initialization.\\

\noindent Fig.~\ref{divergence} demonstrates that the relaxed initialization~(RI) effectively reduces the divergence term and improves the test accuracy. In all tests, when $\beta=0$~(green curve), it represents the vanilla method without RI. After incorporating the RI, the test accuracy achieves at least 2\% improvement on each setup. As introduced in Algorithm~\ref{algorithm} Line.6, a negative $\beta$ means to adopt the relaxed initialization which is close to the latest local model. This is not our expectation. When $\beta=0$, \textit{FedInit} degrades to vanilla \textit{FedAvg}. The positive $\beta$ is the truly effective selection of RI. When we select the negative $\beta$, in most cases it will result in greater divergence and worse test accuracy (orange and blue curves in Fig.~\ref{divergence}). It validates that RI is required to be far away from the local model~(a positive $\beta$). When $\beta$ is small, the correction is limited. The local divergence term is difficult to be diminished efficiently. While it becomes too large, the local training begins from a bad initialization, which can not receive enough guidance of global information from the global models. Furthermore, if the initialization is too far from the local model, the quality of the initialization state will not be effectively guaranteed. Another interesting phenomenon is that RI will cause a large fluctuation of divergence in the early stage of training, but the test accuracy will be improved stably. We speculate that the main reason for this phenomenon is the imprecise selection of $\beta$. One potential research direction is to find more adaptive $\beta$ in the training process. For most existing tasks, constant $\beta$ is already an excellent solution in practice. Furthermore, due to the limited range of $\beta$, it is really easy to search for a optimal value.

\subsection{Communication Rounds and Wall-clock Time Costs}
\begin{table*}[t]
  \caption{Averaged wall-clock time cost~(s$/$round) of the ResNet-18-GN model in the training process on CIFAR-10 dataset.}
  \vspace{0.1cm}
  \label{wall_clock}
  \centering
  \small
  \begin{tabular}{c|c|cccccc}
    \toprule
    & FedAvg & FedAdam & FedSAM & SCAFFOLD & FedDyn & FedCM & FedInit  \\
    \midrule
    10$\%$-100  & 19.38 & 23.22 & 30.23 & 28.61 & 23.84 & 22.63 & \textbf{20.41} \\
    ratio & 1$\times$ & 1.19$\times$ & 1.56$\times$ & 1.47$\times$ & 1.23$\times$ & 1.17$\times$ & \textbf{1.05$\times$}  \\
    \midrule
    5$\%$-200  & 15.87 & 17.50 & 22.18 & 24.49 & 20.61 & 18.19 & \textbf{16.14}  \\
    ratio & 1$\times$ & 1.10$\times$ & 1.40$\times$ & 1.54$\times$ & 1.30$\times$ & 1.15$\times$ & \textbf{1.02$\times$} \\
    \bottomrule
  \end{tabular}
\end{table*}
\begin{table*}[t]
\centering
\renewcommand{\arraystretch}{1.1}
\caption{Acceleration ratio of the communication rounds and wall-clock time of the ResNet-18-GN model on CIFAR-10 dataset.}
\vspace{-0.1cm}
\label{efficiency}
\begin{tabular}{@{}lcc|cc@{}}
\toprule
\multicolumn{1}{c}{} & \multicolumn{2}{c}{Round} & \multicolumn{2}{c}{Time~(s)}\\ 
\cmidrule(lr){2-3} \cmidrule(lr){4-5}
\multicolumn{1}{c}{} & \multicolumn{1}{c}{}  & Speed Ratio & \multicolumn{1}{c}{}  & Speed Ratio \\ 
\cmidrule(lr){1-5}               
FedAvg              & \multicolumn{1}{c}{371} & 1$\times$ & \multicolumn{1}{c}{7189} & 1$\times$ \\ 
FedAdam             & \multicolumn{1}{c}{489} & 0.76$\times$ & \multicolumn{1}{c}{11354} & 0.63$\times$ \\
FedSAM              & \multicolumn{1}{c}{377} & 0.98$\times$ & \multicolumn{1}{c}{11396} & 0.63$\times$ \\ 
SCAFFOLD            & \multicolumn{1}{c}{248} & 1.50$\times$ & \multicolumn{1}{c}{7095} & 1.01$\times$ \\ 
FedDyn              & \multicolumn{1}{c}{192} & 1.93$\times$ & \multicolumn{1}{c}{4577} & 1.57$\times$ \\ 
FedCM               & \multicolumn{1}{c}{183} & 2.02$\times$ & \multicolumn{1}{c}{4141} & 1.73$\times$ \\ 
\textbf{FedInit}    & \multicolumn{1}{c}{\textbf{172}} & \textbf{2.15$\times$} & \multicolumn{1}{c}{\textbf{3510}} & \textbf{2.04$\times$} \\
\midrule
FedAvg              & \multicolumn{1}{c}{191} & 1$\times$ & \multicolumn{1}{c}{3701} & 1$\times$ \\ 
FedAdam             & \multicolumn{1}{c}{256} & 0.74$\times$ & \multicolumn{1}{c}{5944} & 0.62$\times$ \\
FedSAM              & \multicolumn{1}{c}{204} & 0.93$\times$ & \multicolumn{1}{c}{6166} & 0.60$\times$ \\ 
SCAFFOLD            & \multicolumn{1}{c}{211} & 0.90$\times$ & \multicolumn{1}{c}{6036} & 0.61$\times$ \\ 
FedDyn              & \multicolumn{1}{c}{122} & 1.56$\times$ & \multicolumn{1}{c}{2908} & 1.27$\times$ \\ 
FedCM               & \multicolumn{1}{c}{\textbf{95}} & \textbf{2.01}$\times$ & \multicolumn{1}{c}{\textbf{2149}} & \textbf{1.72$\times$} \\ 
\textbf{FedInit}    & \multicolumn{1}{c}{132} & 1.44$\times$ & \multicolumn{1}{c}{2694} & 1.37$\times$ \\
\bottomrule
\end{tabular}
\end{table*}
In this part, we mainly demonstrate the test on time costs, including both the communication rounds and the wall-clock time required. As the major consideration to validate whether an algorithm is practical, the time consumption is one of the very important studies. We verify that although some previous advanced algorithms appear to be efficient, they introduce a large amount of extra calculations and the final running time does not be effectively accelerated as they proposed. 
\\

\noindent As shown in Table~\ref{wall_clock}, due to the additional calculation costs, the practical wall-clock time is different for each method. Generally, \textit{FedAvg} adopts the local-SGD updates without any additional calculations. \textit{FedAdam} adopts similar local-SGD updates and an adaptive optimizer on the global server. \textit{FedSAM} calculation double gradients, which is the main reason for being slowest among the benchmarks. \textit{SCAFFOLD}, \textit{FedDyn}, and \textit{FedCM} are required to calculate some additional vectors to correct the local updates. Therefore they need some additional time costs. Our proposed \textit{FedInit} only adopts an additional initialization calculation, which requires about the same costs as the vanilla \textit{FedAvg}.\\

\noindent Table~\ref{efficiency} shows the communication rounds and wall-clock time required to achieve the target accuracy. In this part, we set the target accuracy and compare their required communication rounds and training time respectively. We test on the ResNet-18-GN model with the setup of 10\%-100 Dir-0.1 splitting. We clearly see that some advanced methods, i.e. \textit{SCAFFOLD} and \textit{FedDyn}, are efficient on the communication round $T$. However, due to the additional costs of each training iteration, they must spend more time on the total training. \textit{FedInit} is a very light and practical method, which only adopts a relaxed initialization on the \textit{FedAvg} method, which makes it to be better and even achieves SOTA results. 

\subsection{Communication Bottleneck and Storage Costs}
In this part, we mainly compare the communication, calculation, and storage costs theoretically and experimentally. By assuming the total model maintains $d$ dimensions, we summarize the costs of benchmarks and our proposed \textit{FedInit} in Table~\ref{storage}. we can see that \textit{SCAFFOLD} and \textit{FedCM} both require double communication costs than the vanilla \textit{FedAvg}. They adopt the correction term~(variance reduction and client-level momentum) to revise each local iteration. Though this achieves good performance, we must indicate that under the millions of edge devices in the FL paradigm, this may introduce a very heavy communication bottleneck. In addition, the \textit{FedSAM} method considers adopting the local SAM optimizer instead of ERM to approach the flat minimal. However, it requires double gradient calculations per iteration. For the very large model, it brings a large calculation cost that can not be neglected. \textit{SCAFFOLD} and \textit{FedDyn} are required to store $2\times$ vectors on each local devices. This is also a limitation for light devices, i.e. mobiles.

\begin{table}[t]
  \caption{Communication and storage costs.}
  \vspace{0.1cm}
  \label{storage}
  \centering
  \renewcommand{\arraystretch}{1.1}
  \setlength{\tabcolsep}{3mm}{\begin{tabular}{c|cr|cr}
    \toprule
    Method & Communication & ratio & Storage & ratio\\
    \midrule
    FedAvg    & $Nd$  & 1$\times$ & $Cd$ & 1$\times$ \\
    FedAdam   & $Nd$  & 1$\times$ & $2Cd$ & 2$\times$ \\
    FedSAM    & $Nd$  & 1$\times$ & $2Cd$ & 2$\times$ \\
    SCAFFOLD  & $2Nd$ & 2$\times$ & $2Cd$ & 2$\times$ \\
    FedDyn    & $Nd$  & 1$\times$ & $2Cd$ & 2$\times$ \\
    FedCM     & $2Nd$ & 2$\times$ & $2Cd$ & 2$\times$ \\
    FedInit   & $Nd$  & 1$\times$ & $Cd$ & 1$\times$ \\
    \bottomrule
  \end{tabular}}
  \vspace{0.1cm}
  {\\$N$: number of participating clients; $C$: number of total clients.}
  \vspace{-0.1cm}
\end{table}

\section{Conclusion}
\label{conclusion}

In this work, we propose an efficient and novel FL method, dubbed \textit{FedInit}, which adopts the stage-wise personalized relaxed initialization~(RI) to improve the generalization efficiency in FL. Furthermore, to clearly understand the essential impact of consistency in FL, we explore the joint analysis of optimization and generalization in FL. Our proofs indicate that consistency dominates RI could help to reduce both the optimization errors and generalization errors. Extensive experiments are conducted to validate the efficiency of relaxed initialization. As a practical and light plug-in, it could also be easily incorporated into other FL paradigms to further improve their performance.




\newpage
\appendix
\clearpage

\section{\bf Proof of Optimization}
\label{app:opt}

\subsection{\bf Some Notations}
We assume the objective function is $L$-smooth and non-convex w.r.t $w$. We could upper bound the training error in the FL. Some useful notations in the proof are introduced in Table~\ref{appendix:tb_notation_for_optimization}. And, some important lemmas are stated as follows.
\begin{table}[h]
  \caption{Some abbreviations of the used terms in the proof of the optimization process.}
  \label{appendix:tb_notation_for_optimization}
  \centering
  \scalebox{1}{
  \begin{tabular}{ccc}
    \toprule
    Notation     &   Formulation   & Description \\
    \midrule
    $w_{i,k}^t$  &  -  & parameters at $k$-th iteration in round $t$ on client $i$ \\
    $w^t$        &  -  & global parameters in round $t$ \\
    $V_1^t$ & $\frac{1}{C}\sum_{i\in\mathcal{C}}\sum_{k=0}^{K-1}\mathbb{E}\Vert w_{i,k}^t - w^t\Vert^2$ & averaged norm of the local updates in round $t$\\
    $V_2^t$ & $\mathbb{E}\Vert w^{t+1} - w^t\Vert^2$ & norm of the global updates in round $t$\\
    $\Delta^t$   & $\frac{1}{C}\sum_{i\in\mathcal{C}}\mathbb{E}\Vert w_{i,K}^{t-1} - w^t\Vert^2$ & inconsistency/divergence term in round $t$\\
    $D$     & $f(w^0)-f(w^\star)$   & bias between the initialization state and optimal \\
    \bottomrule
  \end{tabular}
  }
\end{table}
\subsection{\bf Proofs with Assumption~\ref{assumption:stochastic}}
\subsubsection{\bf Some Important Lemmas}
\begin{lemma}
\label{appendix:bounded_local_updates}
\label{bounded_local}
    {\rm (Bounded local updates)} Under Assumption~\ref{assumption:smooth}, \ref{assumption:stochastic} and \ref{assumption:heterogeneity}, the averaged norm of the local updates of total $C$ clients could be bounded as:
    \begin{equation}
        V_1^t \leq 4K\beta^2\Delta^t + 3K^2\eta^2\left(\sigma_l^2+6K\sigma_g^2\right) + 18K^3\eta^2\mathbb{E}\Vert\nabla f(w^t)\Vert^2.
    \end{equation}
\end{lemma}
\begin{proof}
    $V_1$ measures the norm of the local offset during the local training stage. It could be bounded by two major steps. Firstly, we bound the separated term on the single client $i$ at iteration $k$ as:
    \begin{align*}
        &\quad \ \mathbb{E}_t\Vert w^t-w_{i,k+1}^t\Vert^2 \\
        &= \mathbb{E}_t\Vert w^t-w_{i,k}^t+\eta(g_{i,k}^t-\nabla f_i(w_{i,k}^t)+\nabla f_i(w_{i,k}^t)-\nabla f_i(w^t) + \nabla f_i(w^t) - \nabla f(w^t) + \nabla f(w^t))\Vert^2\\
        &\leq \eta^2\mathbb{E}_t\Vert g_{i,k}^t-\nabla f_i(w_{i,k}^t)\Vert^2 + \left(1+\frac{1}{2K-1}\right)\mathbb{E}_t\Vert w^t-w_{i,k}^t\Vert^2 + 6K\eta^2\mathbb{E}_t\Vert \nabla f_i(w_{i,k}^t)-\nabla f_i(w^t)\Vert^2\\
        &\quad + 6K\eta^2\mathbb{E}_t\Vert \nabla f_i(w^t) - \nabla f(w^t)\Vert^2 + 6K\eta^2\Vert\nabla f(w^t)\Vert^2\\
        &\leq\left(1 + \frac{1}{2K-1} + 6\eta^2KL^2\right)\mathbb{E}_t\Vert w^t-w_{i,k}^t\Vert^2 + \eta^2\sigma_l^2 + 6K\eta^2\sigma_g^2 + 6K\eta^2\Vert\nabla f(w^t)\Vert^2\\
        &\leq \left(1+\frac{1}{K-1}\right)\mathbb{E}_t\Vert w^t-w_{i,k-1}^t\Vert^2 + \eta^2\sigma_l^2 + 6K\eta^2\sigma_g^2 + 6K\eta^2\Vert\nabla f(w^t)\Vert^2,
    \end{align*}
    where the learning rate is required $\eta\leq\frac{1}{\sqrt{6(K-1)(2K-1)}L}$ for $K\geq 2$.\\
    Therefore, by computing the average of the separated term on client $i$:
    \begin{align*}
        \frac{1}{C}\sum_{i\in\mathcal{C}}\mathbb{E}_t\Vert w^t-w_{i,k+1}^t\Vert^2
        &\leq \left(1+\frac{1}{K-1}\right)\frac{1}{C}\sum_{i\in\mathcal{C}}\mathbb{E}_t\Vert w^t-w_{i,k}^t\Vert^2 + \eta^2\sigma_l^2 + 6K\eta^2\sigma_g^2 + 6K\eta^2\Vert\nabla f(w^t)\Vert^2.
    \end{align*}
    Unrolling the aggregated term on iteration $k\leq K$. When local interval $K\geq 2$, $\left(1+\frac{1}{K-1}\right)^k\leq\left(1+\frac{1}{K-1}\right)^K\leq 4$. Then we have:
    \begin{align*}
        &\quad \ \frac{1}{C}\sum_{i\in\mathcal{C}}\mathbb{E}_t\Vert w^t-w_{i,k+1}^t\Vert^2 \\
        &\leq \sum_{\tau=0}^{k}\left(1 + \frac{1}{K-1}\right)^\tau\bigg(\eta^2\sigma_l^2 + 6K\eta^2\sigma_g^2 + 6K\eta^2\Vert\nabla f(w^t)\Vert^2\bigg)+ \left(1+\frac{1}{K-1}\right)^k\frac{1}{C}\sum_{i\in\mathcal{C}}\Vert w^t-w_{i,0}^{t}\Vert^2\\
        &\leq 3(K-1)\bigg(\eta^2\sigma_l^2 + 6K\eta^2\sigma_g^2 + 6K\eta^2\Vert\nabla f(w^t)\Vert^2 \bigg) + 4\beta^2\frac{1}{C}\sum_{i\in\mathcal{C}}\mathbb{E}_t\Vert w^t-w_{i,K}^{t-1}\Vert^2\\
        &< 4\beta^2\Delta^t + 3K\eta^2\left(\sigma_l^2+6K\sigma_g^2\right) + 18K^2\eta^2\Vert\nabla f(w^t)\Vert^2.
    \end{align*}
    Summing the iteration on $k=0,1,\cdots,K-1$, 
    \begin{align*}
        V_1^t = \frac{1}{C}\sum_{i\in\mathcal{C}}\sum_{k=0}^{K-1}\mathbb{E}_t\Vert w^t-w_{i,k}^t\Vert^2\leq 4K\beta^2\Delta^t + 3K^2\eta^2\left(\sigma_l^2+6K\sigma_g^2\right) + 18K^3\eta^2\Vert\nabla f(w^t)\Vert^2.
    \end{align*}
    This completes the proof.
\end{proof}

\begin{lemma}
\label{bounded_global}
\label{appendix:bounded_global_update}
    {\rm (Bounded global updates)} Under Assumption~\ref{assumption:smooth}, \ref{assumption:stochastic} and \ref{assumption:heterogeneity}, let $\eta\leq\frac{1}{KL}$, the norm of the global update of selected $N$ clients could be bounded as:
    \begin{equation}
        \begin{split}
            V_2^t
            &\leq \frac{22\beta^2}{N}\Delta^t + \frac{13\eta^2K}{N}\sigma_l^2 + \frac{76\eta^2K^2}{N}\sigma_g^2 + \frac{76\eta^2K^2}{N}\mathbb{E}\Vert\nabla f(w^t)\Vert^2 + \frac{2\eta^2}{C^2}\mathbb{E}\Vert\sum_{i\in\mathcal{C}}\sum_{k=0}^{K-1}\nabla f_i(w_{i,k}^{t})\Vert^2.
        \end{split}
    \end{equation}
\end{lemma}
\begin{proof}
    $V_2$ measures the variance of the global offset after each communication round. We define an indicator function $\mathbb{I}_{event} = 1$ if the event happens. Then, to upper bound it, we first split the expectation term as:
    \begin{align*}
        &\quad \ \mathbb{E}\Vert w^{t+1}-w^t\Vert^2 = \mathbb{E}\Vert \frac{1}{N}\sum_{i\in\mathcal{N}} w_{i,K}^{t}-w^t\Vert^2\\
        &= \frac{1}{N^2}\mathbb{E}\Vert \sum_{i\in\mathcal{N}} (w_{i,K}^{t}-w^t)\Vert^2 = \frac{1}{N^2}\mathbb{E}\Vert \sum_{i\in\mathcal{C}} (w_{i,K}^{t}-w^t)\mathbb{I}_{i\in\mathcal{N}}\Vert^2\\
        &= \frac{1}{N^2}\mathbb{E}\Vert \sum_{i\in\mathcal{C}}\mathbb{I}_{i\in\mathcal{N}} \left[\sum_{k=0}^{K-1}\eta g_{i,k}^{t} + \beta(w^t-w_{i,K}^{t-1})\right]\Vert^2\\
        &= \frac{\eta^2}{NC}\sum_{i\in\mathcal{C}}\sum_{k=0}^{K-1}\mathbb{E}\Vert g_{i,k}^{t} - \nabla f_i(w_{i,k}^{t})\Vert^2 + \frac{1}{N^2}\mathbb{E}\Vert \sum_{i\in\mathcal{C}}\mathbb{I}_{i\in\mathcal{N}}\left[\sum_{k=0}^{K-1}\eta\nabla f_i(w_{i,k}^{t}) + \beta(w^t-w_{i,K}^{t-1})\right]\Vert^2\\
        &\leq \frac{\eta^2K\sigma_l^2}{N} + \frac{1}{N^2}\mathbb{E}\Vert \sum_{i\in\mathcal{C}}\mathbb{I}_{i\in\mathcal{N}}\left[\sum_{k=0}^{K-1}\eta\nabla f_i(w_{i,k}^{t}) + \beta(w^t-w_{i,K}^{t-1})\right]\Vert^2.
    \end{align*}
    For the second term, we can adopt the following equation. For the arbitrary vector $x_i\in\mathbb{R}^d$,
    \begin{align*}
        \mathbb{E}\Vert\sum_{i\in\mathcal{C}}\mathbb{I}_{i\in\mathcal{N}} x_i\Vert^2
        &= \mathbb{E}\langle\sum_{i\in\mathcal{C}}\mathbb{I}_{i\in\mathcal{N}} x_i, \sum_{j\in\mathcal{C}}\mathbb{I}_{j\in\mathcal{N}} x_j\rangle\\
        &= \sum_{(i\neq j)\in\mathcal{C}}\mathbb{E}\langle\mathbb{I}_{i\in\mathcal{N}} x_i, \mathbb{I}_{j\in\mathcal{N}} x_j\rangle + \sum_{(i=j)\in\mathcal{C}}\mathbb{E}\langle\mathbb{I}_{i\in\mathcal{N}} x_i, \mathbb{I}_{j\in\mathcal{N}} x_j\rangle\\
        &= \sum_{(i\neq j)\in\mathcal{C}}\mathbb{E}\langle\mathbb{I}_{i\in\mathcal{N}} x_i, \mathbb{I}_{j\in\mathcal{N}} x_j\rangle + \sum_{(i=j)\in\mathcal{C}}\mathbb{E}\langle\mathbb{I}_{i\in\mathcal{N}} x_i, \mathbb{I}_{j\in\mathcal{N}} x_j\rangle\\
        &= \frac{N(N-1)}{C(C-1)}\sum_{(i\neq j)\in\mathcal{C}}\mathbb{E}\langle x_i, x_j\rangle + \frac{N}{C}\sum_{(i=j)\in\mathcal{C}}\mathbb{E}\langle x_i, x_j\rangle\\
        &= \frac{N(N-1)}{C(C-1)}\sum_{i,j\in\mathcal{C}}\mathbb{E}\langle x_i, x_j\rangle + \frac{N(C-N)}{C(C-1)}\sum_{(i=j)\in\mathcal{C}}\mathbb{E}\langle x_i, x_j\rangle\\
        &= \frac{N(N-1)}{C(C-1)}\mathbb{E}\Vert\sum_{i\in\mathcal{C}} x_i\Vert^2 + \frac{N(C-N)}{C(C-1)}\sum_{i\in\mathcal{C}}\mathbb{E}\Vert x_i\Vert^2.
    \end{align*}
    We first upper bound the first term. By taking $x_i=\sum_{k=0}^{K-1}\eta\nabla f_i(w_{i,k}^{t}) + \beta(w^t-w_{i,K}^{t-1})$ into $\mathbb{E}\Vert\sum_{i\in\mathcal{C}} x_i\Vert^2$, we have:
    \begin{align*}
        \mathbb{E}\Vert\sum_{i\in\mathcal{C}} \left[\sum_{k=0}^{K-1}\eta\nabla f_i(w_{i,k}^{t}) + \beta(w^t-w_{i,K}^{t-1})\right]\Vert^2
        &\leq 2\eta^2\mathbb{E}\Vert\sum_{i\in\mathcal{C}}\sum_{k=0}^{K-1}\nabla f_i(w_{i,k}^{t})\Vert^2 + 2\beta^2C^2\Delta^t.
    \end{align*}
    Then we upper bound the second term. By taking $x_i=\sum_{k=0}^{K-1}\eta\nabla f_i(w_{i,k}^{t}) + \beta(w^t-w_{i,K}^{t-1})$ into $\sum_{i\in\mathcal{C}}\mathbb{E}\Vert x_i\Vert^2$, we have:
    \begin{align*}
        &\quad \ \sum_{i\in\mathcal{C}}\mathbb{E}\Vert \sum_{k=0}^{K-1}\eta\nabla f_i(w_{i,k}^{t}) + \beta(w^t-w_{i,K}^{t-1})\Vert^2\\
        &= \sum_{i\in\mathcal{C}}\mathbb{E}\Vert \sum_{k=0}^{K-1}\left[\eta\nabla f_i(w_{i,k}^{t}) + \frac{\beta}{K}(w^t-w_{i,K}^{t-1})\right]\Vert^2\\
        &\leq K\sum_{i\in\mathcal{C}}\sum_{k=0}^{K-1}\mathbb{E}\Vert\eta\nabla f_i(w_{i,k}^{t}) + \frac{\beta}{K}(w^t-w_{i,K}^{t-1})\Vert^2\\
        &= K\sum_{i\in\mathcal{C}}\sum_{k=0}^{K-1}\mathbb{E}\Vert\eta\nabla f_i(w_{i,k}^{t}) - \eta\nabla f_i(w^{t}) + \eta\nabla f_i(w^{t}) - \eta\nabla f(w^{t}) + \eta\nabla f(w^{t}) + \frac{\beta}{K}(w^t-w_{i,K}^{t-1})\Vert^2\\
        &\leq 4C\eta^2KL^2V_1^t + 4C\beta^2\Delta^t + 4C\eta^2K^2\sigma_g^2 + 4C\eta^2K^2\mathbb{E}\Vert\nabla f(w^t)\Vert^2.
    \end{align*}
    Let $1\leq N < C$, we have:
    \begin{align*}
        V_2^t
        &\leq \frac{\eta^2K\sigma_l^2}{N} + \frac{1}{N^2}\mathbb{E}\Vert \sum_{i\in\mathcal{C}}\mathbb{I}_{i\in\mathcal{N}}\left[\sum_{k=0}^{K-1}\eta\nabla f_i(w_{i,k}^{t}) + \beta(w^t-w_{i,K}^{t-1})\right]\Vert^2\\
        &\leq \frac{\eta^2K\sigma_l^2}{N} + \frac{4(C-N)}{N(C-1)} (\eta^2KL^2V_1^t + \beta^2\Delta^t + \eta^2K^2\sigma_g^2 + \eta^2K^2\mathbb{E}\Vert\nabla f(w^t)\Vert^2)\\
        &\quad + \frac{2(N-1)}{CN(C-1)}\eta^2\mathbb{E}\Vert\sum_{i\in\mathcal{C}}\sum_{k=0}^{K-1}\nabla f_i(w_{i,k}^{t})\Vert^2 + \frac{2\beta^2(N-1)}{N(C-1)}\Delta^t\\
        &\leq \frac{\eta^2K\sigma_l^2}{N} + \frac{4}{N}\left(\beta^2\Delta^t + \eta^2K^2\sigma_g^2 + \eta^2K^2\mathbb{E}\Vert\nabla f(w^t)\Vert^2\right) + \frac{2\eta^2}{C^2}\mathbb{E}\Vert\sum_{i\in\mathcal{C}}\sum_{k=0}^{K-1}\nabla f_i(w_{i,k}^{t})\Vert^2\\
        &\quad + \frac{4\eta^2K^2L^2}{N}\left(4\beta^2\Delta^t + 3K\eta^2\left(\sigma_l^2+6K\sigma_g^2\right) + 18\eta^2K^2\Vert\nabla f(w^t)\Vert^2\right) + \frac{2\beta^2}{N}\Delta^t\\
        &=\frac{2\beta^2}{N}\left(3 + 8\eta^2K^2L^2\right)\Delta^t + \frac{\eta^2K}{N}\left(1 + 12\eta^2K^2L^2\right)\sigma_l^2 + \frac{4\eta^2K^2}{N}\left(1 + 18\eta^2K^2L^2\right)\sigma_g^2\\
        &\quad + \frac{4\eta^2K^2}{N}\left(1 + 18\eta^2K^2L^2\right)\mathbb{E}\Vert\nabla f(w^t)\Vert^2 + \frac{2\eta^2}{C^2}\mathbb{E}\Vert\sum_{i\in\mathcal{C}}\sum_{k=0}^{K-1}\nabla f_i(w_{i,k}^{t})\Vert^2. 
    \end{align*}
    Generally, we can simplify the coefficient by selecting some special term $\eta^2K^2L^2$. For convenience, we directly select the $\eta\leq \frac{1}{KL}$. This completes the proof.
\end{proof}

\begin{lemma}
\label{bounded_divergence}
    {\rm (Bounded divergence term)} Under Assumption~\ref{assumption:smooth}, \ref{assumption:stochastic} and \ref{assumption:heterogeneity}, let the learning rate satisfy $\eta\leq\frac{1}{KL}$ and let $\kappa_\beta=\frac{1}{1-141\beta^2}$ be a constant, the divergence term $\Delta^t$ could be bounded as the recursion of:
    \begin{equation}
    \begin{split}
        \Delta^t 
        &\leq \kappa_\beta\left(\Delta^t - \Delta^{t+1}\right) + 96\eta^2\kappa_\beta K\left(\sigma_l^2 + 6K\sigma_g^2\right) + 576\kappa_\beta\eta^2K^2\mathbb{E}\Vert\nabla f(w^t)\Vert^2 + \frac{6\kappa_\beta\eta^2}{C^2}\mathbb{E}\Vert\sum_{i\in\mathcal{C}}\sum_{k=0}^{K-1}\nabla f_i(w_{i,k}^{t})\Vert^2.
    \end{split}
    \end{equation}
\end{lemma}
\begin{proof}
    According to the local updates, we have the following recursive formula:
    \begin{align*}
         \underbrace{w^{t+1} - w_{i,K}^{t}}_{\textit{local bias in round $t+1$}} = \beta\underbrace{(w_{i,K}^{t-1} - w^t)}_{\textit{local bias in round $t$}}  + (w^{t+1} -w^t) + \sum_{k=0}^{K-1}\eta g_{i,k}^t.
    \end{align*}
    By taking the squared norm and expectation on both sides, we have:
    \begin{align*}
        \mathbb{E}\Vert w^{t+1} - w_{i,K}^{t}\Vert^2
        &= \mathbb{E}\Vert \beta(w_{i,K}^{t-1} - w^t) + w^{t+1} - w^t + \sum_{k=0}^{K-1}\eta g_{i,k}^t\Vert^2\\
        &\leq 3\beta^2\mathbb{E}\Vert w_{i,K}^{t-1} - w^t\Vert^2 + 3\underbrace{\mathbb{E}\Vert w^{t+1} -w^t\Vert^2}_{V_2^t} + 3\mathbb{E}\Vert\sum_{k=0}^{K-1}\eta g_{i,k}^t\Vert^2.
    \end{align*}
    The second term in the above inequality is $V_2$ we have bounded in lemma~\ref{appendix:bounded_global_update}. Then we bound the stochastic gradients term. We have:
    \begin{align*}
        \mathbb{E}\Vert\sum_{k=0}^{K-1}\eta g_{i,k}^t\Vert^2
        &\leq 2\eta^2\mathbb{E}\Vert\sum_{k=0}^{K-1} \left(g_{i,k}^t - \nabla f_i(w_{i,k}^t)\right)\Vert^2 + 2\eta^2\mathbb{E}\Vert\sum_{k=0}^{K-1} \nabla f_i(w_{i,k}^t)\Vert^2\\
        &\leq 2\eta^2K\sigma_{l}^2 + 2\eta^2K\sum_{k=0}^{K-1}\mathbb{E}\Vert \nabla f_i(w_{i,k}^t) -\nabla f_i(w^t) + \nabla f_i(w^t) -\nabla f(w^t) + \nabla f(w^t)\Vert^2\\
        &\leq 2\eta^2K\sigma_{l}^2 + 6\eta^2KL^2\sum_{k=0}^{K-1}\mathbb{E}\Vert w_{i,k}^t-w^t\Vert^2 + 6\eta^2K^2\sigma_g^2 + 6\eta^2K^2\mathbb{E}\Vert\nabla f(w^t)\Vert^2.
    \end{align*}
    Taking the average on client $i$, we have:
    \begin{align*}
        \frac{1}{C}\sum_{i\in\mathcal{C}}\mathbb{E}\Vert\sum_{k=0}^{K-1}\eta g_{i,k}^t\Vert^2
        &\leq 2\eta^2K\sigma_{l}^2 + 6\eta^2KL^2V_1^t + 6\eta^2K^2\sigma_g^2 + 6\eta^2K^2\mathbb{E}\Vert\nabla f(w^t)\Vert^2.
    \end{align*}
    With $\eta\leq\frac{1}{KL}$ and combining this and the squared norm inequality, we have:
    \begin{align*}
        \Delta^{t+1}
        &=\frac{1}{C}\sum_{i\in\mathcal{C}}\mathbb{E}\Vert w^{t+1} - w_{i,K}^{t}\Vert^2 \leq 3\beta^2\Delta^t + 3V_2^t + \frac{3}{C}\sum_{i\in\mathcal{C}}\mathbb{E}\Vert\sum_{k=0}^{K-1}\eta g_{i,k}^t\Vert^2\\
        &\leq 3\beta^2\left(1+\frac{22}{N}+24\eta^2K^2L^2\right)\Delta^t + 3\eta^2K\left(1+\frac{13}{N}+18\eta^2K^2L^2\right)\sigma_l^2\\
        &\quad + 6\eta^2K^2\left(3+\frac{39}{N}+54\eta^2K^2L^2\right)\sigma_g^2 + 6\eta^2K^2\left(3+\frac{39}{N}+54\eta^2K^2L^2\right)\mathbb{E}\Vert\nabla f(w^t)\Vert^2 \\
        &\quad + \frac{6\eta^2}{C^2}\mathbb{E}\Vert\sum_{i\in\mathcal{C}}\sum_{k=0}^{K-1}\nabla f_i(w_{i,k}^{t})\Vert^2.
    \end{align*}
    Similarly, we can simplify the constant as:
    \begin{align*}
        \Delta^{t+1}\leq 141\beta^2\Delta^t + 96\eta^2K\left(\sigma_l^2 + 6K\sigma_g^2\right) + 576\eta^2K^2\mathbb{E}\Vert\nabla f(w^t)\Vert^2 + \frac{6\eta^2}{C^2}\mathbb{E}\Vert\sum_{i\in\mathcal{C}}\sum_{k=0}^{K-1}\nabla f_i(w_{i,k}^{t})\Vert^2.
    \end{align*}
    Let $141\beta^2 < 1$, thus we add $(1-141\beta^2)\Delta^t$ on both sides and get the recursive formulation:
    \begin{align*}
        (1-141\beta^2)\Delta^t 
        &\leq (\Delta^t - \Delta^{t+1}) + 96\eta^2K\left(\sigma_l^2 + 6K\sigma_g^2\right) + 576\eta^2K^2\mathbb{E}\Vert\nabla f(w^t)\Vert^2 + \frac{6\eta^2}{C^2}\mathbb{E}\Vert\sum_{i\in\mathcal{C}}\sum_{k=0}^{K-1}\nabla f_i(w_{i,k}^{t})\Vert^2.
    \end{align*}
    Then we multiply the $\frac{1}{1-141\beta^2}$ on both sides, which completes the proof.
\end{proof}
\quad
\subsubsection{\bf Expanding the Smoothness Iteration}
For the non-convex and $L$-smooth function, we firstly expand the smoothness inequality at round $t$ as:
\begin{align*}
    &\quad \ \mathbb{E}[f(w^{t+1}) - f(w^t)] \\
    &\leq \mathbb{E}\langle\nabla f(w^t),w^{t+1}-w^t\rangle+\frac{L}{2}\underbrace{\mathbb{E}\Vert w^{t+1}-w^t\Vert^2}_{V_2^t}\\
    &= \mathbb{E}\langle\nabla f(w^t),\frac{1}{N}\sum_{i\in\mathcal{N}} w_{i,K}^t- w^t\rangle+\frac{LV_2^t}{2}\\
    &= \mathbb{E}\langle\nabla f(w^t),\frac{1}{C}\sum_{i\in\mathcal{C}}\left[( w_{i,K}^t-w_{i,0}^t)+\beta(w^{t}-w_{i,K}^{t-1})\right]\rangle+\frac{LV_2^t}{2}\\
    &= - \eta\mathbb{E}\langle\nabla f(w^t),\frac{1}{C}\sum_{i\in\mathcal{C}}\sum_{k=0}^{K-1} \nabla f_i(w_{i,k}^t) -\frac{1}{C}\sum_{i\in\mathcal{C}}\sum_{k=0}^{K-1} \nabla f_i(w^t) + K\nabla f(w^t)\rangle +\frac{LV_2^t}{2}\\
    &= - \eta K\mathbb{E}\Vert f(w^t)\Vert^2 + \mathbb{E}\langle\sqrt{\eta K}\nabla f(w^t),\sqrt{\frac{\eta}{K}}\frac{1}{C}\sum_{i\in\mathcal{C}}\sum_{k=0}^{K-1}\left(\nabla f_i(w^t)- \nabla f_i(w_{i,k}^t) \right)\rangle +\frac{LV_2^t}{2}\\
    &\leq - \eta K\mathbb{E}\Vert f(w^t)\Vert^2 + \frac{\eta K}{2}\mathbb{E}\Vert f(w^t)\Vert^2 + \frac{\eta}{2C}\sum_{i\in\mathcal{C}}\sum_{k=0}^{K-1}\mathbb{E}\Vert\nabla f_i(w^t)- \nabla f_i(w_{i,k}^t)\Vert^2\\
    &\quad - \frac{\eta}{2C^2K}\mathbb{E}\Vert\sum_{i\in\mathcal{C}}\sum_{k=0}^{K-1}\nabla f_i(w_{i,k}^t)\Vert^2 + \frac{LV_2^t}{2}\\
    &\leq - \frac{\eta K}{2}\mathbb{E}\Vert f(w^t)\Vert^2 + \frac{\eta L^2}{2}\underbrace{\frac{1}{C}\sum_{i\in\mathcal{C}}\sum_{k=0}^{K-1}\mathbb{E}\Vert w^t- w_{i,k}^t\Vert^2}_{V_1^t}- \frac{\eta}{2C^2K}\mathbb{E}\Vert\sum_{i\in\mathcal{C}}\sum_{k=0}^{K-1}\nabla f_i(w_{i,k}^t)\Vert^2 + \frac{LV_2^t}{2}\\
    &\leq - \frac{\eta K}{2}\mathbb{E}\Vert f(w^t)\Vert^2 + \frac{\eta L^2V_1^t}{2} - \frac{\eta}{2C^2K}\mathbb{E}\Vert\sum_{i\in\mathcal{C}}\sum_{k=0}^{K-1}\nabla f_i(w_{i,k}^t)\Vert^2 + \frac{LV_2^t}{2}.
\end{align*}
According to Lemma~\ref{appendix:bounded_local_updates} and lemma~\ref{appendix:bounded_global_update} to bound the $V_1^t$ and $V_2^t$, we can get the following recursive formula:
\begin{align*}
    &\quad \ \mathbb{E}[f(w^{t+1}) - f(w^t)] \\
    &\leq -\frac{\eta K}{2}\mathbb{E}\Vert f(w^t)\Vert^2 + \frac{\eta L^2V_1^t}{2} - \frac{\eta}{2C^2K}\mathbb{E}\Vert\sum_{i\in\mathcal{C}}\sum_{k=0}^{K-1}\nabla f_i(w_{i,k}^t)\Vert^2 + \frac{LV_2^t}{2}\\
    &\leq -\frac{\eta K}{2}\mathbb{E}\Vert f(w^t)\Vert^2 + \frac{\eta L^2}{2}\left[4K\beta^2\Delta^t + 3K^2\eta^2\left(\sigma_l^2+6K\sigma_g^2\right) + 18K^3\eta^2\mathbb{E}\Vert\nabla f(w^t)\Vert^2\right]\\
    &\quad + \frac{L}{2}\left[\frac{22\beta^2}{N}\Delta^t + \frac{13\eta^2K}{N}\sigma_l^2 + \frac{76\eta^2K^2}{N}\sigma_g^2 + \frac{76\eta^2K^2}{N}\mathbb{E}\Vert\nabla f(w^t)\Vert^2\right]\\
    &\quad + \frac{\eta^2 L}{C^2}\mathbb{E}\Vert\sum_{i\in\mathcal{C}}\sum_{k=0}^{K-1}\nabla f_i(w_{i,k}^{t})\Vert^2 - \frac{\eta}{2C^2K}\mathbb{E}\Vert\sum_{i\in\mathcal{C}}\sum_{k=0}^{K-1}\nabla f_i(w_{i,k}^t)\Vert^2\\
    &\leq -\left(\frac{\eta}{2C^2K} - \frac{\eta^2L}{C^2}\right)\mathbb{E}\Vert\sum_{i\in\mathcal{C}}\sum_{k=0}^{K-1}\nabla f_i(w_{i,k}^{t})\Vert^2 -\eta K\left(\frac{1}{2}-\frac{38\eta KL}{N}-9\eta^2K^2L^2\right)\mathbb{E}\Vert\nabla f(w^t)\Vert^2\\
    &\quad + \frac{\eta^2 KL}{2N}\left(13 + 3\eta NKL\right)\left(\sigma_l^2 + 6K\sigma_g^2\right) + \frac{\beta^2 L}{N}\left(11 + 2\eta NKL\right)\Delta^t.
\end{align*}
We can also select some special cases on learning rate to simplify the above formula. In fact, in lemma~\ref{bounded_local}, there is a constraint on the learning rate as $\eta \leq \frac{1}{\sqrt{6(K-1)(2K-1)}L}$ for $K\geq 2$. In lemma~\ref{bounded_global} and lemma~\ref{bounded_divergence}, there is a constraint on the learning rate as $\eta\leq\frac{1}{KL}$. To further simplify the coefficients, we select $\eta\leq \frac{1}{NKL}$ to remove the constant parts. Therefore,
\begin{align*}
    &\quad \ \mathbb{E}[f(w^{t+1}) - f(w^t)] \\
    &\leq -\frac{\eta}{C^2}\left(\frac{1}{2K} - \eta L - \frac{78\beta^2\kappa_\beta\eta K^2L}{N}\right)\mathbb{E}\Vert\sum_{i\in\mathcal{C}}\sum_{k=0}^{K-1}\nabla f_i(w_{i,k}^{t})\Vert^2 + \frac{\eta^2 \kappa KL}{N}\left(\sigma_l^2 + 6K\sigma_g^2\right)\\
    &\quad -\eta K\left(\frac{1}{2} - \frac{38\eta KL}{N} - \frac{728\beta^2\kappa_\beta \eta KL}{N} - 9\eta^2K^2L^2\right)\mathbb{E}\Vert\nabla f(w^t)\Vert^2 + \frac{13\beta^2\kappa_\beta L}{N}\left(\Delta^t - \Delta^{t+1}\right),
\end{align*}
where $\kappa = 8 + 78\beta^2\kappa_\beta^2$ is the constant coefficient.\\\\
\noindent 
We follow \cite{karimireddy2020scaffold,yang2021achieving} and let $\frac{1}{2K} - \eta L - \frac{78\beta^2\kappa_\beta\eta K^2L}{N}\geq 0$ which indicates the learning rate $\eta\leq\frac{1}{2KL(1+\frac{78\beta^2\kappa_\beta K^2}{N})}\leq\frac{1}{2KL}$. Then, according to the study of \cite{yang2021achieving}, the $\frac{1}{2} - \frac{38\eta KL}{N} - \frac{728\beta^2\kappa_\beta \eta KL}{N} - 9\eta^2K^2L^2$ term could be bounded as a constant $\lambda\in(0,\frac{1}{2})$. Therefore, we have:
\begin{align*}
    \lambda\eta K\mathbb{E}\Vert f(w^t)\Vert^2
    &\leq \mathbb{E}[f(w^{t}) - f(w^{t+1})] + \frac{\eta^2\kappa KL}{N}(\sigma_l^2 + 6K\sigma_g^2) + \frac{13\beta^2\kappa_\beta L}{N}\left(\Delta^t - \Delta^{t+1}\right).
\end{align*}
Taking the accumulation from $0$ to $T-1$ and let $D=\mathbb{E}\left[f(w^{0}) - f(w^{\star})\right]$ be the initial bias, we have:
\begin{align*}
    \frac{1}{T}\sum_{t=0}^{T-1}\mathbb{E}\Vert f(w^t)\Vert^2
    &\leq \frac{\mathbb{E}\left[f(w^{0}) - f(w^{T})\right]}{\lambda \eta KT} + \eta\frac{\kappa L}{\lambda N}(\sigma_l^2 + 6K\sigma_g^2) + \frac{13\beta^2\kappa_\beta L}{\lambda N\eta KT}\left(\Delta^0 - \Delta^{T}\right)\\
    &\leq \frac{1}{\lambda \eta KT}D + \frac{\eta\kappa L}{\lambda N}(\sigma_l^2 + 6K\sigma_g^2) - \frac{13\beta^2\kappa_\beta L}{\lambda N\eta KT}\Delta^{T}.
\end{align*}
The general analysis of the convergence only includes the impacts of the stochastic variance $\sigma_l^2$, and heterogeneity variance $\sigma_g^2$. However, $FedInit$ could benefit from the relaxed initial by the $\Delta^T$ term after $T$ communication rounds.\\

\subsubsection{\bf Expanding the Smoothness Iteration under P\L-condition}
Assumption~\ref{assumption:PL} indicates the {\it P\L}-condition to measure the true loss difference instead of the gradient norm. Similarly, according to the expansion of the smoothness inequality above
and Assumption~\ref{assumption:PL}, we have $\mu (f(w)-f(w^\star))\leq\Vert\nabla f(w)\Vert^2$, then:
\begin{align*}
    &\quad \ \lambda\mu\eta K \mathbb{E}[f(w^t)-f(w^\star)]\leq \lambda\eta K\mathbb{E}\Vert f(w^t)\Vert^2 \\
    &\leq\mathbb{E}[f(w^{t}) - f(w^{t+1})] + \frac{\eta^2\kappa KL}{N}(\sigma_l^2 + 6K\sigma_g^2) + \frac{13\beta^2\kappa_\beta L}{N}\left(\Delta^t - \Delta^{t+1}\right).
\end{align*}
Combining the terms aligned with $w^t$ and $w^{t+1}$, we have:
\begin{align*}
    \mathbb{E}[f(w^{t+1})-f(w^\star)]
    &\leq (1-\lambda\mu\eta K)\mathbb{E}[f(w^{t})-f(w^\star)] + \frac{\eta^2\kappa KL}{N}(\sigma_l^2 + 6K\sigma_g^2) + \frac{13\beta^2\kappa_\beta L}{N}\left(\Delta^t - \Delta^{t+1}\right).
\end{align*}
Taking the recursion from $t=0$ to $T-1$ and let learning rate $\eta\leq\frac{1}{\lambda\mu K}$, we have:
\begin{align*}
    &\quad \ \mathbb{E}[f(w^{T})-f(w^\star)] \\
    &\leq (1-\lambda\mu\eta K)^T\mathbb{E}[f(w^{0})-f(w^\star)] + \frac{13\beta^2\kappa_\beta L}{N}\sum_{t=0}^{T-1}(1-\lambda\mu\eta K)^{T-1-t}(\Delta^t - \Delta^{t+1})\\ 
    &\quad + \frac{\eta^2\kappa KL}{N}(\sigma_l^2 + 6K\sigma_g^2)\sum_{t=0}^{T-1}(1-\lambda\mu\eta K)^{T-1-t}\\
    &= e^{-\lambda\mu\eta KT}\mathbb{E}[f(w^{0})-f(w^\star)] + \frac{\eta^2\kappa KL}{N}(\sigma_l^2 + 6K\sigma_g^2)\frac{1-(1-\lambda\mu\eta K)^T}{\lambda\mu\eta K}\\ 
    &\quad + \frac{13\beta^2\kappa_\beta L}{N}\left(\sum_{t=1}^{T-1}\left[(1-\lambda\mu\eta K)^{T-1-t} - (1-\lambda\mu\eta K)^{T-t}\right]\Delta^t + (1-\lambda\mu\eta K)^{T-1}\Delta^0-\Delta^T\right)\\
    &< e^{-\lambda\mu\eta KT}D + \frac{\eta\kappa L}{\lambda\mu N}(\sigma_l^2 + 6K\sigma_g^2) + \frac{13\beta^2\kappa_\beta L}{N}\sum_{t=0}^{T-1}\lambda\mu\eta K(1-\lambda\mu\eta K)^{T-1-t}\Delta^t\\
    &\leq e^{-\lambda\mu\eta KT}D + \frac{\eta\kappa L}{\lambda\mu N}(\sigma_l^2 + 6K\sigma_g^2) + \frac{13\beta^2\kappa_\beta\lambda\mu \eta KTL}{N}\frac{1}{T}\sum_{t=1}^{T-1}\Delta^t.
\end{align*}
According to Lemma~\ref{bounded_divergence}, we have:
\begin{align*}
    \Delta^t 
    &\leq \kappa_\beta\left(\Delta^t - \Delta^{t+1}\right) + 96\eta^2\kappa_\beta K\left(\sigma_l^2 + 6K\sigma_g^2\right) + 576\kappa_\beta\eta^2K^2\mathbb{E}\Vert\nabla f(w^t)\Vert^2 + \frac{6\kappa_\beta\eta^2}{C^2}\mathbb{E}\Vert\sum_{i\in\mathcal{C}}\sum_{k=0}^{K-1}\nabla f_i(w_{i,k}^{t})\Vert^2.
\end{align*}
Here we first bound the gradient term as:
\begin{align*}
    \mathbb{E}\Vert\sum_{i\in\mathcal{C}}\sum_{k=0}^{K-1}\nabla f_i(w_{i,k}^{t})\Vert^2
    &\leq \mathbb{E}\Vert\sum_{i\in\mathcal{C}}\sum_{k=0}^{K-1}\left(\nabla f_i(w_{i,k}^{t}) - \nabla f_i(w^t) + \nabla f(w^{t})\right)\Vert^2\\
    &\leq 2C^2KL^2V_1^t + 2C^2K^2\mathbb{E}\Vert\nabla f(w^{t})\Vert^2.
\end{align*}
Combining this into the recursive formulation and let $C_\beta=\frac{\kappa_\beta}{1-48\beta^2\kappa_\beta}$ and $\eta KL\leq 1$,
\begin{align*}
    \Delta^{t}
    &\leq C_\beta\left(\Delta^t - \Delta^{t+1}\right) + 132C_\beta\eta^2 K\left(\sigma_l^2 + 6K\sigma_g^2\right) + 804C_\beta\eta^2K^2\mathbb{E}\Vert\nabla f(w^t)\Vert^2.
\end{align*}
By taking the accumulation from $t=0$ to $T-1$ and according to the bound of gradients,
\begin{align*}
    \frac{1}{T}\sum_{t=0}^{T-1}\Delta^{t}
    &\leq \frac{C_\beta}{T}\left(\Delta^0 - \Delta^{T}\right) + 132C_\beta\eta^2 K\left(\sigma_l^2 + 6K\sigma_g^2\right) + 804C_\beta\eta^2K^2\frac{1}{T}\sum_{t=0}^{T-1}\mathbb{E}\Vert\nabla f(w^t)\Vert^2\\
    &\leq \eta\frac{804C_\beta K}{\lambda T}D + C_\beta\eta^2 K\left(132 + \frac{804\kappa}{\lambda N}\right)\left(\sigma_l^2 + 6K\sigma_g^2\right)= \mathcal{O}\left(\frac{\eta}{T}+\eta^2\right).
\end{align*}
Therefore, by combining the :
\begin{align*}
    \mathbb{E}[f(w^{T})-f(w^\star)]
    &\leq e^{-\lambda\mu\eta KT}D + \eta\frac{\kappa L}{\lambda\mu N}(\sigma_l^2 + 6K\sigma_g^2) + \mathcal{O}\left(\frac{\eta}{T}+\eta^2\right).
\end{align*}

\subsection{\bf Proofs under Interpolation Conditions}
\subsubsection{\bf Some Important Lemmas}
\begin{lemma}
\label{interpolation:bounded_local}
    {\rm (Bounded local updates)} Under Assumption~\ref{assumption:smooth}, \ref{assumption:interpolation}, and \ref{assumption:global_interpolation}, the averaged norm of the local updates of total $C$ clients could be bounded as:
    \begin{equation}
        V_1^t\leq 4K\beta^2\Delta^t + 12a^2 b^2\eta^2K^3\mathbb{E}\Vert\nabla f(w^t)\Vert^2.
    \end{equation}
\end{lemma}
\begin{proof}
    According to Lemma~\ref{appendix:bounded_local_updates}, we adopt the interpolation and have:
    \begin{align*}
        &\quad \ \mathbb{E}_t\Vert w^t - w_{i,k}^t\Vert^2 = \mathbb{E}_t \Vert w^t - w_{i,k-1}^t + \eta g_{i,k-1}^t\Vert^2\\
        &\leq \left(1+\frac{1}{2K-1}\right)\mathbb{E}_t\Vert w^t - w_{i,k-1}^t\Vert^2 + 2\eta^2K\mathbb{E}_t\Vert g_{i,k-1}^t\Vert^2\\
        &\leq \left(1+\frac{1}{2K-1}\right)\mathbb{E}_t\Vert w^t - w_{i,k-1}^t\Vert^2 + 2\eta^2Ka^2\Vert\nabla f_{i}(w_{i,k-1}^t) - f_{i}(w^t) + f_{i}(w^t)\Vert^2\\
        &\leq \left(1+\frac{1}{2K-1} + 4a^2\eta^2KL^2\right)\mathbb{E}_t\Vert w^t - w_{i,k-1}^t\Vert^2 + 4\eta^2Ka^2 b^2\Vert\nabla f(w)\Vert^2\\
        &\leq \left(1+\frac{1}{K-1}\right)\mathbb{E}_t\Vert w^t - w_{i,k-1}^t\Vert^2 + 4a^2 b^2\eta^2K\Vert\nabla f(w)\Vert^2,
    \end{align*}
    where the learning rate is required as $\eta\leq\frac{1}{2a\sqrt{(2K-1)(K-1)}L}$ for $K\geq 2$.\\
    Similarly, we have:
    \begin{align*}
        \frac{1}{C}\sum_{i\in\mathcal{C}}\mathbb{E}_t\Vert w^t - w_{i,k+1}^t\Vert^2\leq \left(1+\frac{1}{K-1}\right)\frac{1}{C}\sum_{i\in\mathcal{C}}\mathbb{E}_t\Vert w^t - w_{i,k}^t\Vert^2 + 4a^2 b^2\eta^2K\Vert\nabla f(w)\Vert^2.
    \end{align*}
    Unrolling the index $k$, we have:
    \begin{align*}
        \frac{1}{C}\sum_{i\in\mathcal{C}}\mathbb{E}_t\Vert w^t - w_{i,k+1}^t\Vert^2
        &\leq 12a^2 b^2\eta^2K^2\Vert\nabla f(w)\Vert^2 + 4\beta^2\Delta^t.
    \end{align*}
    Summing the iteration from $k=0$ to $K-1$ completes the proofs.
\end{proof}

\begin{lemma}
\label{interpolation:bounded_global}
    {\rm(Bounded global updates)} Under Assumption~\ref{assumption:smooth}, \ref{assumption:interpolation}, and \ref{assumption:global_interpolation}, the norm of the global update of selected $N$ clients could be bounded as:
    \begin{equation}
        V_2^t\leq 4\beta^2\Delta^t + 12a^2b^2\eta^2K^2\Vert\nabla f(w)\Vert^2.
    \end{equation}
\end{lemma}
\begin{proof}
    According to Lemma~\ref{appendix:bounded_global_update}, we modify the local iterations as:
    \begin{align*}
        \mathbb{E}\Vert w^{t+1}-w^t\Vert^2 
        &= \mathbb{E}\Vert \frac{1}{N}\sum_{i\in\mathcal{N}} w_{i,K}^{t}-w^t\Vert^2\leq \frac{1}{N}\mathbb{E}\sum_{i\in\mathcal{N}}\Vert  (w_{i,K}^{t}-w^t)\Vert^2 \leq 12a^2 b^2\eta^2K^2\Vert\nabla f(w)\Vert^2 + 4\beta^2\Delta^t.
    \end{align*}
    This completes the proofs.
\end{proof}

\begin{lemma}
\label{interpolation:bounded_divergence}
    {\rm(Bounded divergence term)} Under Assumption~\ref{assumption:smooth}, \ref{assumption:interpolation}, and \ref{assumption:global_interpolation}, let the learning rate $\eta\leq\frac{1}{aKL}$ and $\gamma_\beta=\frac{1}{1-39\beta^2}$, the divergence term satisfies:
    \begin{equation}
        \Delta^t \leq \gamma_\beta\left(\Delta^t - \Delta^{t+1}\right) + 114\gamma_\beta a^2b^2\eta^2K^2\mathbb{E}\Vert\nabla f(w^t)\Vert^2.
    \end{equation}
\end{lemma}
\begin{proof}
    Similarly, we adopt the recursive formula and upper bound the gradient term as:
    \begin{align*}
        \mathbb{E}\Vert\sum_{k=0}^{K-1}\eta g_{i,k}^t\Vert^2
        &= \eta^2\mathbb{E}\Vert\sum_{k=0}^{K-1} g_{i,k}^t\Vert^2 \leq \eta^2K\sum_{k=0}^{K-1}\mathbb{E}\Vert g_{i,k}^t\Vert^2\leq a^2\eta^2K\sum_{k=0}^{K-1}\mathbb{E}\Vert \nabla f_i(w_{i,k}^t)\Vert^2\\
        &\leq 2a^2\eta^2KL^2\sum_{k=0}^{K-1}\mathbb{E}\Vert w_{i,k}^t - w^t\Vert^2 + 2a^2b^2\eta^2K^2\mathbb{E}\Vert \nabla f(w^t)\Vert^2.
    \end{align*}
    Therefore, the same let $\eta \leq \frac{1}{aKL}$, we have:
    \begin{align*}
        \frac{1}{C}\sum_{i\in\mathcal{C}}\mathbb{E}\Vert\sum_{k=0}^{K-1}\eta g_{i,k}^t\Vert^2
        &\leq 2a^2\eta^2KL^2V_1^t + 2a^2b^2\eta^2K^2\mathbb{E}\Vert\nabla f(w^t)\Vert^2 \leq 8\beta^2\Delta^t + 26a^2b^2\eta^2K^2\mathbb{E}\Vert\nabla f(w^t)\Vert^2.
    \end{align*}
    According to the recursive formula of the divergence term in Lemma~\ref{bounded_divergence},
    \begin{align*}
        \Delta^{t+1}
        &=\frac{1}{C}\sum_{i\in\mathcal{C}}\mathbb{E}\Vert w^{t+1} - w_{i,K}^t\Vert^2\leq 3\beta^2\Delta^t + 3V_2^2 + \frac{3}{C}\sum_{i\in\mathcal{C}}\mathbb{E}\Vert\sum_{k=0}^{K-1}\eta g_{i,k}^t\Vert^2\\
        &\leq 3\beta^2\Delta^t + 3(4\beta^2\Delta^t + 12a^2b^2\eta^2K^2\Vert\nabla f(w)\Vert^2) + 3(8\beta^2\Delta^t + 26a^2b^2\eta^2K^2\mathbb{E}\Vert\nabla f(w^t)\Vert^2)\\
        &\leq 39\beta^2\Delta^t + 114a^2b^2\eta^2K^2\mathbb{E}\Vert\nabla f(w^t)\Vert^2.
    \end{align*}
    Let $39\beta^2\leq 1$ and $\gamma_\beta=\frac{1}{1-39\beta^2}$, we can bounded the recursive formulation as:
    \begin{align*}
        \Delta^t \leq \gamma_\beta\left(\Delta^t - \Delta^{t+1}\right) + 114\gamma_\beta a^2b^2\eta^2K^2\mathbb{E}\Vert\nabla f(w^t)\Vert^2.
    \end{align*}
    This completes the proofs.
\end{proof}

\subsubsection{\bf Expanding the Smoothness Iteration}
Similarly, we first expand the smoothness as:
\begin{align*}
    \mathbb{E}[f(w^{t+1})-f(w^t)] \leq - \frac{\eta K}{2}\mathbb{E}\Vert f(w^t)\Vert^2 + \frac{\eta L^2V_1^t}{2} - \frac{\eta}{2C^2K}\mathbb{E}\Vert\sum_{i\in\mathcal{C}}\sum_{k=0}^{K-1}\nabla f_i(w_{i,k}^t)\Vert^2 + \frac{LV_2^t}{2}.
\end{align*}
Combining the Lemma~\ref{interpolation:bounded_local}, \ref{interpolation:bounded_global}, and \ref{interpolation:bounded_divergence}, we have:
\begin{align*}
    &\quad \ \mathbb{E}[f(w^{t+1})-f(w^t)]\leq - \frac{\eta K}{2}\mathbb{E}\Vert f(w^t)\Vert^2 + \frac{\eta L^2V_1^t}{2} + \frac{LV_2^t}{2}\\
    &\leq - \frac{\eta K}{2}\mathbb{E}\Vert f(w^t)\Vert^2 + \frac{\eta L^2}{2}(4K\beta^2\Delta^t + 12a^2 b^2\eta^2K^3\mathbb{E}\Vert\nabla f(w^t)\Vert^2) + \frac{L}{2}(4\beta^2\Delta^t + 12a^2b^2\eta^2K^2\Vert\nabla f(w)\Vert^2)\\
    &\leq -\eta K\left(\frac{1}{2}-6a^2b^2\eta KL-6a^2b^2\eta^2K^2L^2\right)\mathbb{E}\Vert\nabla f(w^t)\Vert^2 \\
    &\quad + 2\beta^2L\left[\gamma_\beta\left(\Delta^t - \Delta^{t+1}\right) + 114\gamma_\beta a^2b^2\eta^2K^2\mathbb{E}\Vert\nabla f(w^t)\Vert^2\right]\\
    &= 2\gamma_\beta\beta^2L\left(\Delta^t - \Delta^{t+1}\right) -\eta K\left(\frac{1}{2}-6a^2b^2\eta KL-114\gamma_\beta a^2b^2\eta K-6a^2b^2\eta^2K^2L^2\right)\mathbb{E}\Vert\nabla f(w^t)\Vert^2.
\end{align*}
Let $\eta\leq\frac{1}{KL}$ and $\zeta=\frac{1}{2}-6a^2b^2\eta KL-114\gamma_\beta a^2b^2\eta K-6a^2b^2\eta^2K^2L^2$ is a constant within $(0,\frac{1}{2})$, we have:
\begin{align*}
    \mathbb{E}\Vert\nabla f(w^t)\Vert^2 \leq \frac{\mathbb{E}[f(w^{t})-f(w^{t+1})]}{\zeta\eta K} + \frac{2\gamma_\beta\beta^2L}{\zeta\eta K}\left(\Delta^t - \Delta^{t+1}\right).
\end{align*}
Taking the accumulation from $0$ to $T-1$ and let $D=\mathbb{E}\left[f(w^0)-f(w^\star)\right]$ be the initial bias, we have:
\begin{align*}
    \frac{1}{T}\sum_{t=0}^{T-1}\mathbb{E}\Vert\nabla f(w^t)\Vert^2 
    &\leq \frac{\mathbb{E}\left[f(w^0)-f(w^T)\right]}{\zeta\eta KT} + \frac{2\gamma_\beta\beta^2L}{\zeta\eta KT}\left(\Delta^0 - \Delta^T\right)\leq \frac{D}{\zeta\eta KT} - \frac{2\gamma_\beta\beta^2L}{\zeta\eta KT}\Delta^T.
\end{align*}

\subsubsection{\bf Expanding the Smoothness Iteration under P\L-condition}
Similarly, we adopt the Assumption~\ref{assumption:PL} in the smoothness property as:
\begin{align*}
    \zeta\mu\eta K\mathbb{E}[f(w^t)-f(w^\star)]\leq \zeta\eta K\mathbb{E}\Vert\nabla f(w^t)\Vert^2 \leq \mathbb{E}[f(w^t)-f(w^{t+1})] + 2\gamma_\beta\beta^2L\left(\Delta^t - \Delta^{t+1}\right).
\end{align*}
We have:
\begin{align*}
    \mathbb{E}[f(w^t+1)-f(w^\star)]\leq (1-\zeta\mu\eta K)\mathbb{E}[f(w^t)-f(w^\star)] + 2\gamma_\beta\beta^2L\left(\Delta^t - \Delta^{t+1}\right).
\end{align*}
Taking the recursion from $t=0$ to $T-1$, we have:
\begin{align*}
    \mathbb{E}[f(w^T)-f(w^\star)]
    &\leq (1-\zeta\mu\eta K)^T + 2\gamma_\beta\beta^2L\sum_{t=0}^{T-1}(1-\zeta\mu\eta K)^{T-1-t}\left(\Delta^t-\Delta^{t+1}\right)\\
    &\leq e^{-\zeta\mu\eta KT}D + 2\gamma_\beta\beta^2\zeta\mu\eta KTL\left(\frac{1}{T}\sum_{t=0}^{T-1}\Delta^t\right).
\end{align*}
According to Lemma~\ref{interpolation:bounded_divergence}, we have:
\begin{align*}
    \Delta^t \leq \gamma_\beta\left(\Delta^t - \Delta^{t+1}\right) + 114\gamma_\beta a^2b^2\eta^2K^2\mathbb{E}\Vert\nabla f(w^t)\Vert^2.
\end{align*}
By taking the accumulation from $t=0$ to $T-1$, 
\begin{align*}
    \frac{1}{T}\sum_{t=0}^{T-1}\Delta^t
    &\leq \frac{\gamma_\beta}{T}\left(\Delta^0 - \Delta^{T}\right) + 114\gamma_\beta a^2b^2\eta^2K^2\left(\frac{1}{T}\sum_{t=0}^{T-1}\mathbb{E}\Vert\nabla f(w^t)\Vert^2\right)\leq \frac{114\gamma_\beta a^2b^2\eta K}{\zeta T}D.
\end{align*}
Therefore, by combining the above two inequalities we have:
\begin{align*}
    \mathbb{E}[f(w^T)-f(w^\star)]\leq e^{-\zeta\mu\eta KT}D + 228\mu\gamma_\beta^2\beta^2a^2b^2\eta^2K^2LD.
\end{align*}

\section{\bf Proof of Generalization}
\label{app:gen}
\subsection{\bf Some Notations}
We still let the objective be a smooth and non-convex finite-sum function. We could upper bound the stability bias in the FL. We first introduce the proof background. According to the stability analysis, we suppose there are $C$ clients participating in the training process as a set $\mathcal{C}=\{i\}_{i=1}^C$. Each client has a local dataset $\mathcal{S}_i=\{z_j\}_{j=1}^S$ with total $S$ data sampled from a specific unknown distribution $\mathcal{D}_i$. Now we define a re-sampled dataset $\widetilde{\mathcal{S}}_i$ which only differs from the dataset $\mathcal{S}_i$ on the $j^\star$-th data. We replace the $\mathcal{S}_{i^\star}$ with $\widetilde{\mathcal{S}}_{i^\star}$ and keep other $C-1$ local dataset, which composes a new set $\widetilde{\mathcal{C}}$. From the perspective of total data, $\mathcal{C}$ only differs from the $\widetilde{\mathcal{C}}$ at $j^\star$-th data on the $i^\star$-th client. Then, based on these two sets, our method could generate two output models, $w^t$ and $\widetilde{w}^t$ respectively, after $t$ training rounds. We first introduce some notations used in the proof of the generalization error.
\begin{table}[ht]
  \caption{Some abbreviations of the used terms in the proof of bounded stability error.}
  \label{appendix:tb_notation_for_generalization}
  \centering
  \begin{tabular}{ccc}
    \toprule
    Notation     &   Formulation   & Description \\
    \midrule
    $w$  &  -  & parameters trained with set $\mathcal{C}$ \\
    $\widetilde{w}$  &  -  & parameters trained with set $\widetilde{\mathcal{C}}$ \\
    $\Delta^t$   & $\frac{1}{C}\sum_{i\in\mathcal{C}}\mathbb{E}\Vert w_{i,K}^{t-1} - w^t\Vert^2$ & inconsistency/divergence term in round $t$\\
    $\delta_k^t$  &  $\frac{1}{C}\sum_{i\in\mathcal{C}}\mathbb{E}\Vert w_{i,k}^{t} - \widetilde{w}_{i,k}^t\Vert^2$ & stability difference at $k$-iteration on $t$-round\\
    \bottomrule
  \end{tabular}
\end{table}

\subsection{\bf Proofs with Assumption~\ref{assumption:stochastic}}
\subsubsection{\bf Some Important Lemmas}
\begin{lemma}
\label{appendix:stability}
    {\rm(Lemma 3.11 of \cite{hardt2016train})} We follow the definition of \cite{hardt2016train,zhou2021towards} to upper bound the uniform stability term after each communication round in FL paradigm. Different from their vanilla calculations, FL considers the finite-sum function on heterogeneous clients. Let non-negative objective $f$ is $L$-smooth and $L_G$-Lipschitz. After training $T$ rounds on $\mathcal{C}$ and $\widetilde{\mathcal{C}}$, our method generates two models $w^{T+1}$ and $\widetilde{w}^{T+1}$ respectively. For each data $z$ and every $t_0\in\{1,2,3,\cdots,S\}$, we have:
    \begin{equation}
        \mathbb{E}\Vert f(w^{T+1};z) - f(\widetilde{w}^{T+1};z)\Vert \leq \frac{NUKt_0}{CS} + L_G\delta_K^T.
    \end{equation}
\end{lemma}
\begin{proof}
    Let $\xi=1$ denote the event $\Vert w^{t_0} - \widetilde{w}^{t_0}\Vert=0$ and $U=\sup_{w,z}f(w;z)$, we have:
    \begin{align*}
        &\quad \ \mathbb{E}\Vert f(w^{T+1};z) - f(\widetilde{w}^{T+1};z)\Vert \\
        &= P(\{\xi\}) \ \mathbb{E}\left[\Vert f(w^{T+1};z) - f(\widetilde{w}^{T+1};z)\Vert \ | \ \xi\right] + P(\{\xi^c\}) \ \mathbb{E}\left[\Vert f(w^{T+1};z) - f(\widetilde{w}^{T+1};z)\Vert \ | \ \xi^c\right]\\
        &\leq \mathbb{E}\left[\Vert f(w^{T+1};z) - f(\widetilde{w}^{T+1};z)\Vert \ | \ \xi\right] + P(\{\xi^c\})\sup_{w,z}f(w;z)\\
        &\leq L_G\mathbb{E}\left[\Vert w^{T+1} - \widetilde{w}^{T+1}\Vert \ | \ \xi\right] + UP(\{\xi^c\})\\
        &= L_G\mathbb{E}\left[\Vert \frac{1}{C}\sum_{i\in\mathcal{C}} (w_{i,K}^T - \widetilde{w}_{i,K}^T)\Vert \ | \ \xi\right] + UP(\{\xi^c\})\\
        &\leq \frac{L_G}{C}\sum_{i\in\mathcal{C}}\mathbb{E}\left[\Vert w_{i,K}^T - \widetilde{w}_{i,K}^T\Vert \ \vert \ \xi \right] + UP(\{\xi^c\}).
    \end{align*}
    Before the $j^\star$-th data on $i^\star$-th client is sampled, the iterative states are identical on both $\mathcal{C}$ and $\widetilde{\mathcal{C}}$. Let $I$ be the index of the first different sampling, if $I > t_0$, then $\xi=1$ holds for $t_0$. Let $\chi$ be the event of selecting $\mathcal{S}_{i^\star}$ dataset and $\tau_0 = t_0K+k_0$ be the observation moment $(t_0, k_0)$. Therefore, we have:
    \begin{align*}
        P(I\leq t_0K+k_0) 
        &\leq \sum_{t=0}^{t_0-1}\sum_{k=0}^{K-1} P(I = tK+k;\chi) + \sum_{k=0}^{k_0}P(I = t_0K+k;\chi)\\
        &= \sum_{t=0}^{t_0-1}\sum_{k=0}^{K-1}\sum_{\chi} P(I = tK+k\vert\chi)P(\chi) + \sum_{k=0}^{k_0}\sum_{\chi} P(I = t_0K+k\vert\chi)P(\chi)\\
        &= \frac{N}{C}\left(\sum_{t=0}^{t_0-1}\sum_{k=0}^{K-1} P(I = tK+k) + \sum_{k=0}^{k_0}P(I = t_0K+k)\right) = \frac{NKt_0}{CS}.
    \end{align*}
    This completes the proof.
\end{proof}

\begin{lemma}
\label{appendix:recursive_with_same_data}
    {\rm (Lemma 1.1 of \cite{zhou2021towards})} Different from their calculations, we prove similar inequalities on $f$ in the stochastic optimization. Under Assumption~\ref{assumption:smooth} and \ref{assumption:stochastic}, the local updates satisfy $w_{i,k+1}^t = w_{i,k}^t - \eta g_{i,k}^t$ on $\mathcal{C}$ and $\widetilde{w}_{i,k+1}^t = \widetilde{w}_{i,k}^t - \eta \widetilde{g}_{i,k}^t$ on $\widetilde{\mathcal{C}}$. If at $k$-th iteration on each round, we sample the {\textbf{\textit{same}}} data in $\mathcal{C}$ and $\widetilde{\mathcal{C}}$, then we have:
    \begin{equation}
        \mathbb{E}\Vert w_{i,k+1}^t - \widetilde{w}_{i,k+1}^t\Vert \leq (1+\eta L)\mathbb{E}\Vert w_{i,k}^t - \widetilde{w}_{i,k}^t\Vert.
    \end{equation}
\end{lemma}
\begin{proof}
    In each round $t$, by the triangle inequality and omitting the same data $z$, we have:
    \begin{align*}
        &\quad \ \mathbb{E}\Vert w_{i,k+1}^t - \widetilde{w}_{i,k+1}^t\Vert= \mathbb{E}\Vert w_{i,k}^t - \eta g_{i,k}^t - \widetilde{w}_{i,k}^t - \eta \widetilde{g}_{i,k}^t\Vert\\
        &\leq \mathbb{E}\Vert w_{i,k}^t - \widetilde{w}_{i,k}^t\Vert +  \eta\mathbb{E}\Vert \nabla f_i(w_{i,k}^t,z) - \nabla f_i(\widetilde{w}_{i,k}^t,z)\Vert \leq (1+\eta L)\mathbb{E}\Vert w_{i,k}^t - \widetilde{w}_{i,k}^t\Vert.
    \end{align*}
    This completes the proof.
\end{proof}

\begin{lemma}
\label{appendix:recursive_with_different_data}
    {\rm (Lemma 1.2 of \cite{zhou2021towards})} Different from their calculations, we prove similar inequalities on $f$ in the stochastic optimization. Under Assumption~\ref{assumption:smooth} and \ref{assumption:stochastic}, the local updates satisfy $w_{i,k+1}^t = w_{i,k}^t - \eta g_{i,k}^t$ on $\mathcal{C}$ and $\widetilde{w}_{i,k+1}^t = \widetilde{w}_{i,k}^t - \eta \widetilde{g}_{i,k}^t$ on $\widetilde{\mathcal{C}}$. If at $k$-th iteration on each round, we sample the {\textbf{\textit{different}}} data in $\mathcal{C}$ and $\widetilde{\mathcal{C}}$, then we have:
    \begin{equation}
        \mathbb{E}\Vert w_{i,k+1}^t - \widetilde{w}_{i,k+1}^t\Vert\leq(1+\eta L)\mathbb{E}\Vert w_{i,k}^t - \widetilde{w}_{i,k}^t\Vert + 2\eta\sigma_l.
    \end{equation}
\end{lemma}
\begin{proof}
    In each round $t$, let  by the triangle inequality and denoting the different data as $z$ and $\widetilde{z}$, we have:
    \begin{align*}
        &\quad \ \mathbb{E}\Vert w_{i,k+1}^t - \widetilde{w}_{i,k+1}^t\Vert\\
        &= \mathbb{E}\Vert w_{i,k}^t - \eta g_{i,k}^t - \widetilde{w}_{i,k}^t - \eta \widetilde{g}_{i,k}^t\Vert\\
        &\leq \mathbb{E}\Vert w_{i,k}^t - \widetilde{w}_{i,k}^t\Vert +  \eta\mathbb{E}\Vert g_{i,k}^t - \widetilde{g}_{i,k}^t\Vert\\
        &= \mathbb{E}\Vert w_{i,k}^t - \widetilde{w}_{i,k}^t\Vert +  \eta\mathbb{E}\Vert g_{i,k}^t - \nabla f_i(w_{i,k}^t) - \widetilde{g}_{i,k}^t - \nabla f_i(\widetilde{w}_{i,k}^t) + \nabla f_i(w_{i,k}^t)- \nabla f_i(\widetilde{w}_{i,k}^t)\Vert\\
        &\leq \mathbb{E}\Vert w_{i,k}^t - \widetilde{w}_{i,k}^t\Vert +  2\eta\sigma_l + \eta\mathbb{E}\Vert\nabla f_i(w_{i,k}^t)- \nabla f_i(\widetilde{w}_{i,k}^t)\Vert\\
        &\leq (1+\eta L)\mathbb{E}\Vert w_{i,k}^t - \widetilde{w}_{i,k}^t\Vert + 2\eta\sigma_l.
    \end{align*}
    This completes the proof.
\end{proof}

\subsubsection{\bf Bounded Uniform Stability}
According to Lemma~\ref{appendix:stability}, we firstly bound the recursive stability on $k$ in one round. If the sampled data is the same, we can adopt Lemma~\ref{appendix:recursive_with_same_data}. Otherwise, we adopt Lemma~\ref{appendix:recursive_with_different_data}. Thus we can bound the second term in Lemma~\ref{appendix:stability} as:
\begin{align*}
    \delta_{k+1}^t
    &= \frac{1}{C}\sum_{i\in\mathcal{C}}\mathbb{E}\left[\Vert w_{i,k+1}^t - \widetilde{w}_{i,k+1}^t\Vert \ \vert \ \xi\right]\\
    &= P(z) \ \frac{1}{C}\sum_{i\in\mathcal{C}}\mathbb{E}\left[\Vert w_{i,k+1}^t - \widetilde{w}_{i,k+1}^t\Vert \ \vert \ \xi,z\right] + P(\widetilde{z}) \ \frac{1}{C}\sum_{i\in\mathcal{C}}\mathbb{E}\left[\Vert w_{i,k+1}^t - \widetilde{w}_{i,k+1}^t\Vert \ \vert \ \xi,\widetilde{z}\right]\\
    &\leq \frac{1}{C}\sum_{i\in\mathcal{C}}(1+\eta L)\mathbb{E}\left[\Vert w_{i,k}^t - \widetilde{w}_{i,k}^t\Vert \ \vert \ \xi\right] + \frac{1}{CS}\left((1+\eta L)\mathbb{E}\left[\Vert w_{i^\star,k}^t - \widetilde{w}_{i^\star,k}^t\Vert \ \vert \ \xi\right] + 2\eta\sigma_l\right) \\
    &\leq \left(1+\eta L\right)\delta_k^t + \frac{2\eta\sigma_l}{CS}.
\end{align*}
Balancing the LHS and RHS, we have the following recursive formulation:
\begin{align*}
    \delta_{k+1}^t + \frac{2\sigma_l}{CSL} \leq (1 + \eta L)\left(\delta_{k+1}^t + \frac{2\sigma_l}{CSL}\right).
\end{align*}
Therefore, in one single communication round, by generally defining learning rate $\eta=\eta_k^t$,
\begin{align*}
    \delta_{K}^t + \frac{2\sigma_l}{CSL} \leq \left(\prod_{k=0}^{K-1}(1 + \eta_k^t L)\right)\left(\delta_{0}^t + \frac{2\sigma_l}{CSL}\right).
\end{align*}
The next important relationship is to measure the $\delta_K^{t-1}$ and $\delta_0^t$. According to the update rule $w_{i,0}^t=w^t+\beta(w^t-w_{i,K}^{t-1})$, we have the difference follows:
\begin{align*}
    w_{i,0}^t - \widetilde{w}_{i,0}^t 
    &= w^t- \widetilde{w}^t + \beta(w^t - w_{i,K}^{t-1}) - \beta(\widetilde{w}^t - \widetilde{w}_{i,K}^{t-1})\\
    & = (1+\beta)(w^t- \widetilde{w}^t) - \beta(w_{i,K}^{t-1} - \widetilde{w}_{i,K}^{t-1})
\end{align*}
By taking the expectation on the $l_2$ norm, we have:
\begin{align*}
    \delta_0^t = \frac{1}{C}\sum_{i\in\mathcal{C}}\mathbb{E}\Vert w_{i,0}^t - \widetilde{w}_{i,0}^t\Vert
    &\leq \frac{1+\beta}{C}\sum_{i\in\mathcal{C}}\mathbb{E}\Vert w^t - \widetilde{w}^t\Vert + \frac{\beta}{C}\sum_{i\in\mathcal{C}}\mathbb{E}\Vert w_{i,K}^{t-1} - \widetilde{w}_{i,K}^{t-1}\Vert \leq (1+2\beta)\delta_K^{t-1}.
\end{align*}
By denoting $\phi(t)=\prod_{k=0}^{K-1}(1 + \eta_k^t L)$ be the combination of learning rate $\eta_k^t$, we can provide an upper bound of the recursive formulation as:
\begin{align*}
    \delta_{K}^t + \frac{2\sigma_l}{CSL} 
    &\leq \left(\prod_{k=0}^{K-1}(1 + \eta_k^t L)\right)\left(\delta_{0}^t + \frac{2\sigma_l}{CSL}\right)\leq \phi(t)\left[\left(1+2\beta\right)\delta_{K}^{t-1} + \frac{2\sigma_l}{CSL}\right].
\end{align*}
To balance the constant part, assuming the learning rate is decayed by communication round $t$ which indicates $\phi(t)\leq \phi(t-1)$ and let $2\beta\leq\frac{\phi(t-1)}{\phi(t)}-1$ be the upper bound, then we have the following recursive formulation:
\begin{align*}
    \delta_{K}^t + \frac{\phi(t)-1}{(1+2\beta)\phi(t)-1}\frac{2\sigma_l}{CSL}
    &\leq \phi(t-1)\left(\delta_{K}^{t-1} + \frac{\phi(t-1)-1}{(1+2\beta)\phi(t-1)-1}\frac{2\sigma_l}{CSL}\right).
\end{align*}
Unrolling from $t_0-1$ to $T$, we have:
\begin{align*}
    \delta_{K}^{T+1}
    &\leq \left(\prod_{\tau=t_0-1}^{T}\phi(\tau)\right)\left(\delta_{K}^{t_0-1} + \frac{\phi(t_0-1)-1}{(1+2\beta)\phi(t_0-1)-1}\frac{2\sigma_l}{CSL}\right) - \frac{\phi(T+1)-1}{(1+2\beta)\phi(T+1)-1}\frac{2\sigma_l}{CSL}\\
    &\leq \left(\prod_{\tau=t_0-1}^{T}\phi(\tau)\right)\frac{\phi(t_0-1)-1}{(1+2\beta)\phi(t_0-1)-1}\frac{2\sigma_l}{CSL} \leq \left(\prod_{\tau=t_0-1}^{T}\prod_{k=0}^{K-1}(1 + \eta_k^\tau L)\right)\frac{2\sigma_l}{(1+2\beta)CSL}\\
    &\leq \exp{\left(\sum_{\tau=t_0-1}^{T}\sum_{k=0}^{K-1}\eta_{k}^{\tau}L\right)}\frac{2\sigma_l}{(1+2\beta)CSL}.
\end{align*}
Let the learning rate be the same selection as it in the optimization of $\mathcal{O}(\frac{1}{t})=\frac{c}{t}$, we have:
\begin{align*}
    \delta_{K}^{T}
    &\leq \exp{\left(\sum_{\tau=t_0-1}^{T-1}\sum_{k=0}^{K-1}\eta_{k}^{\tau}L\right)}\frac{2\sigma_l}{(1+2\beta)CSL} \leq \exp{\left(\sum_{\tau=t_0-1}^{T-1}\frac{cKL}{\tau}\right)}\frac{2\sigma_l}{(1+2\beta)CSL}\\
    &\leq \frac{2\sigma_l}{(1+2\beta)CSL}\exp{\left(\int_{\tau=t_0}^{T}\frac{cKL}{\tau}d\tau\right)} = \frac{2\sigma_l}{(1+2\beta)CSL}\left(\frac{T}{t_0}\right)^{cKL}.
\end{align*}
To summarize the above inequalities and the Lemma~\ref{appendix:stability}, we have:
\begin{align*}
    \mathbb{E}\Vert f(w^{T+1};z) - f(\widetilde{w}^{T+1};z)\Vert 
    &\leq \frac{NUKt_0}{CS} + L_G\delta_K^T \leq \frac{NUKt_0}{CS} + \frac{2\sigma_lL_G}{(1+2\beta)CSL}\left(\frac{T}{t_0}\right)^{cKL}.
\end{align*}
Furthermore, to minimize the stability errors, we can select the proper observation point $t_0=\left[\frac{2\sigma_lL_G}{(1+2\beta)NUKL}\right]^{\frac{1}{1+cKL}}T^{\frac{cKL}{1+cKL}}$ and then we have:
\begin{align*}
    \mathbb{E}\Vert f(w^{T+1};z) - f(\widetilde{w}^{T+1};z)\Vert 
    &\leq \frac{2}{CS}\left[\frac{2\sigma_lL_G}{(1+2\beta)L}\right]^{\frac{1}{1+cKL}}\left(NUKT\right)^{\frac{cKL}{1+cKL}}.
\end{align*}

\subsection{\bf Proofs with Interpolation Conditions}
\subsubsection{\bf Some Important Lemmas}
\begin{lemma}
\label{appendix:interpolation_update}
    Under Assumption~\ref{assumption:smooth} and \ref{assumption:interpolation}, and \ref{assumption:assumption_Lipschitz}, the local updates satisfy $w_{i,k+1}^t = w_{i,k}^t - \eta g_{i,k}^t$ on $\mathcal{C}$ and $\widetilde{w}_{i,k+1}^t = \widetilde{w}_{i,k}^t - \eta \widetilde{g}_{i,k}^t$ on $\widetilde{\mathcal{C}}$. If at $k$-th iteration on each round, we sample the {\textbf{\textit{different}}} data in $\mathcal{C}$ and $\widetilde{\mathcal{C}}$, then we have:
    \begin{equation}
        \mathbb{E}\Vert w_{i,k+1}^t - \widetilde{w}_{i,k+1}^t\Vert\leq \mathbb{E}\Vert w_{i,k}^t - \widetilde{w}_{i,k}^t\Vert + 2ab\eta L_G.
    \end{equation}
\end{lemma}
\begin{proof}
    Similar to Lemma~\ref{appendix:recursive_with_same_data}, we directly upper bound the recursive formulation as:
    \begin{align*}
        \mathbb{E}\Vert w_{i,k+1}^t - \widetilde{w}_{i,k+1}^t\Vert 
        &= \mathbb{E}\Vert w_{i,k}^t - \eta g_{i,k}^t - \widetilde{w}_{i,k}^t - \eta \widetilde{g}_{i,k}^t\Vert \leq \mathbb{E}\Vert w_{i,k}^t - \widetilde{w}_{i,k}^t\Vert +  \eta\mathbb{E}\Vert \nabla f_i(w_{i,k}^t,z) - \nabla f_i(\widetilde{w}_{i,k}^t, \widetilde{z})\Vert\\
        &\leq \mathbb{E}\Vert w_{i,k}^t - \widetilde{w}_{i,k}^t\Vert + 2ab\eta L_G.
    \end{align*}
   This completes the proofs.
\end{proof}

\subsubsection{\bf Bounded Uniform Stability}
Similar to the last section, we need to provide the upper bound of the $\delta_K^T$ term. We first rebuild the recursive formulation according to Lemma~\ref{appendix:interpolation_update} and above conclusion,
\begin{align*}
    \delta_{k+1}^t
    \leq \delta_{k}^t + \frac{ab\eta L_G}{CS}.
\end{align*}
Taking the accumulation from $k=0$ to $K-1$ in a single round and let $\varphi(t)=\sum_{k=0}^{K-1}\eta_{k}^t$, we have:
\begin{align*}
    \delta_{K}^t
    \leq \delta_{0}^t + \frac{abL_G}{CS}\sum_{k=0}^{K-1}\eta_{k}^t \leq (1+2\beta)\delta_{K}^{t-1} + \frac{abL_G}{CS}\varphi(t).
\end{align*}
Similarly, according to the factor of $\varphi(t)\leq \varphi(t-1)$, let $\beta^t$ be related to $t$ and be decayed by the round $t$, further we assume the $\beta$ satisfies $\frac{\varphi{t}}{\beta^t}\leq \frac{\varphi(t-1)}{\beta^{t-1}}$, we have:
\begin{align*}
    \delta_{K}^t + \frac{abL_G}{2\beta^t CS}\varphi(t) \leq (1+2\beta^t)\left(\delta_{K}^{t-1} + \frac{abL_G}{2\beta^t CS}\varphi(t)\right)\leq (1+2\beta^t)\left(\delta_{K}^{t-1} + \frac{abL_G}{2\beta^{t-1} CS}\varphi(t-1)\right).
\end{align*}
Taking the accumulation from $t=1$ to $T-1$ and adopting the factor of $\varphi(0)=cK$, we have:
\begin{align*}
    \delta_{K}^T \leq \left(\prod_{t=1}^{T}(1+2\beta^t)\right)\frac{abL_G}{2\beta^0 CS}\varphi(0) \leq e^{2\sum_{t=1}^{T}\beta^t}\frac{abcL_GK}{2\beta^0 CS}.
\end{align*}
Lemma~\ref{appendix:interpolation_update} does not identify the data sample, therefore we can revise Lemma~\ref{appendix:stability} as:
\begin{align*}
    \mathbb{E}\Vert f(w^{T+1};z) - f(\widetilde{w}^{T+1};z)\Vert \leq L_G\delta_K^T\leq e^{2\sum_{t=1}^{T}\beta^t}\frac{abcL_G^2K}{2\beta^0 CS}.
\end{align*}



\bibliography{sample}

@article{MAL-083,
  title={Advances and open problems in federated learning},
  author={Kairouz, Peter and McMahan, H Brendan and Avent, Brendan and Bellet, Aur{\'e}lien and Bennis, Mehdi and Bhagoji, Arjun Nitin and Bonawitz, Kallista and Charles, Zachary and Cormode, Graham and Cummings, Rachel and others},
  journal={Foundations and Trends{\textregistered} in Machine Learning},
  volume={14},
  number={1--2},
  pages={1--210},
  year={2021},
  publisher={Now Publishers, Inc.}
}

@inproceedings{shi2023make,
  title={Make landscape flatter in differentially private federated learning},
  author={Shi, Yifan and Liu, Yingqi and Wei, Kang and Shen, Li and Wang, Xueqian and Tao, Dacheng},
  booktitle={Proceedings of the IEEE/CVF Conference on Computer Vision and Pattern Recognition},
  pages={24552--24562},
  year={2023}
}

@article{defazio2014saga,
  title={SAGA: A fast incremental gradient method with support for non-strongly convex composite objectives},
  author={Defazio, Aaron and Bach, Francis and Lacoste-Julien, Simon},
  journal={Advances in neural information processing systems},
  volume={27},
  year={2014}
}

@article{johnson2013accelerating,
  title={Accelerating stochastic gradient descent using predictive variance reduction},
  author={Johnson, Rie and Zhang, Tong},
  journal={Advances in neural information processing systems},
  volume={26},
  year={2013}
}

@article{lin2018don,
  title={Don't use large mini-batches, use local sgd},
  author={Lin, Tao and Stich, Sebastian U and Patel, Kumar Kshitij and Jaggi, Martin},
  journal={arXiv preprint arXiv:1808.07217},
  year={2018}
}

@article{woodworth2020minibatch,
  title={Minibatch vs local sgd for heterogeneous distributed learning},
  author={Woodworth, Blake E and Patel, Kumar Kshitij and Srebro, Nati},
  journal={Advances in Neural Information Processing Systems},
  volume={33},
  pages={6281--6292},
  year={2020}
}

@inproceedings{gorbunov2021local,
  title={Local SGD: Unified theory and new efficient methods},
  author={Gorbunov, Eduard and Hanzely, Filip and Richt{\'a}rik, Peter},
  booktitle={International Conference on Artificial Intelligence and Statistics},
  pages={3556--3564},
  year={2021},
  organization={PMLR}
}

@inproceedings{mcmahan2017communication,
  title={Communication-efficient learning of deep networks from decentralized data},
  author={McMahan, Brendan and Moore, Eider and Ramage, Daniel and Hampson, Seth and y Arcas, Blaise Aguera},
  booktitle={Artificial intelligence and statistics},
  pages={1273--1282},
  year={2017},
  organization={PMLR}
}

@article{zhang2021fedpd,
  title={FedPD: A federated learning framework with adaptivity to non-iid data},
  author={Zhang, Xinwei and Hong, Mingyi and Dhople, Sairaj and Yin, Wotao and Liu, Yang},
  journal={IEEE Transactions on Signal Processing},
  volume={69},
  pages={6055--6070},
  year={2021},
  publisher={IEEE}
}

@inproceedings{wang2022fedadmm,
  title={Fedadmm: A federated primal-dual algorithm allowing partial participation},
  author={Wang, Han and Marella, Siddartha and Anderson, James},
  booktitle={2022 IEEE 61st Conference on Decision and Control (CDC)},
  pages={287--294},
  year={2022},
  organization={IEEE}
}

@inproceedings{gong2022fedadmm,
  title={FedADMM: A robust federated deep learning framework with adaptivity to system heterogeneity},
  author={Gong, Yonghai and Li, Yichuan and Freris, Nikolaos M},
  booktitle={2022 IEEE 38th International Conference on Data Engineering (ICDE)},
  pages={2575--2587},
  year={2022},
  organization={IEEE}
}

@article{zhou2023federated,
  title={Federated learning via inexact ADMM},
  author={Zhou, Shenglong and Li, Geoffrey Ye},
  journal={IEEE Transactions on Pattern Analysis and Machine Intelligence},
  year={2023},
  publisher={IEEE}
}

@inproceedings{karimireddy2020scaffold,
  title={Scaffold: Stochastic controlled averaging for federated learning},
  author={Karimireddy, Sai Praneeth and Kale, Satyen and Mohri, Mehryar and Reddi, Sashank and Stich, Sebastian and Suresh, Ananda Theertha},
  booktitle={International Conference on Machine Learning},
  pages={5132--5143},
  year={2020},
  organization={PMLR}
}

@article{yang2021achieving,
  title={Achieving linear speedup with partial worker participation in non-iid federated learning},
  author={Yang, Haibo and Fang, Minghong and Liu, Jia},
  journal={arXiv preprint arXiv:2101.11203},
  year={2021}
}

@article{xu2021fedcm,
  title={Fedcm: Federated learning with client-level momentum},
  author={Xu, Jing and Wang, Sen and Wang, Liwei and Yao, Andrew Chi-Chih},
  journal={arXiv preprint arXiv:2106.10874},
  year={2021}
}

@article{reddi2020adaptive,
  title={Adaptive federated optimization},
  author={Reddi, Sashank and Charles, Zachary and Zaheer, Manzil and Garrett, Zachary and Rush, Keith and Kone{\v{c}}n{\`y}, Jakub and Kumar, Sanjiv and McMahan, H Brendan},
  journal={arXiv preprint arXiv:2003.00295},
  year={2020}
}

@article{acar2021federated,
  title={Federated learning based on dynamic regularization},
  author={Acar, Durmus Alp Emre and Zhao, Yue and Navarro, Ramon Matas and Mattina, Matthew and Whatmough, Paul N and Saligrama, Venkatesh},
  journal={arXiv preprint arXiv:2111.04263},
  year={2021}
}

@article{wang2021local,
  title={Local adaptivity in federated learning: Convergence and consistency},
  author={Wang, Jianyu and Xu, Zheng and Garrett, Zachary and Charles, Zachary and Liu, Luyang and Joshi, Gauri},
  journal={arXiv preprint arXiv:2106.02305},
  year={2021}
}

@article{li2020federated,
  title={Federated optimization in heterogeneous networks},
  author={Li, Tian and Sahu, Anit Kumar and Zaheer, Manzil and Sanjabi, Maziar and Talwalkar, Ameet and Smith, Virginia},
  journal={Proceedings of Machine learning and systems},
  volume={2},
  pages={429--450},
  year={2020}
}

@article{karimi2021layer,
  title={Layer-wise and Dimension-wise Locally Adaptive Federated Learning},
  author={Karimi, Belhal and Li, Ping and Li, Xiaoyun},
  journal={arXiv preprint arXiv:2110.00532},
  year={2021}
}

@article{huang2023fusion,
  title={Fusion of Global and Local Knowledge for Personalized Federated Learning},
  author={Huang, Tiansheng and Shen, Li and Sun, Yan and Lin, Weiwei and Tao, Dacheng},
  journal={arXiv preprint arXiv:2302.11051},
  year={2023}
}

@article{sun2023adasam,
  title={Adasam: Boosting sharpness-aware minimization with adaptive learning rate and momentum for training deep neural networks},
  author={Sun, Hao and Shen, Li and Zhong, Qihuang and Ding, Liang and Chen, Shixiang and Sun, Jingwei and Li, Jing and Sun, Guangzhong and Tao, Dacheng},
  journal={arXiv preprint arXiv:2303.00565},
  year={2023}
}

@inproceedings{qu2022generalized,
  title={Generalized federated learning via sharpness aware minimization},
  author={Qu, Zhe and Li, Xingyu and Duan, Rui and Liu, Yao and Tang, Bo and Lu, Zhuo},
  booktitle={International Conference on Machine Learning},
  pages={18250--18280},
  year={2022},
  organization={PMLR}
}

@article{sun2023fedspeed,
  title={Fedspeed: Larger local interval, less communication round, and higher generalization accuracy},
  author={Sun, Yan and Shen, Li and Huang, Tiansheng and Ding, Liang and Tao, Dacheng},
  journal={arXiv preprint arXiv:2302.10429},
  year={2023}
}

@article{slowmo,
  title={SlowMo: Improving communication-efficient distributed SGD with slow momentum},
  author={Wang, Jianyu and Tantia, Vinayak and Ballas, Nicolas and Rabbat, Michael},
  journal={arXiv preprint arXiv:1910.00643},
  year={2019}
}

@inproceedings{inconsistency1,
  title={Convergence and accuracy trade-offs in federated learning and meta-learning},
  author={Charles, Zachary and Kone{\v{c}}n{\`y}, Jakub},
  booktitle={International Conference on Artificial Intelligence and Statistics},
  pages={2575--2583},
  year={2021},
  organization={PMLR}
}

@inproceedings{inconsistency2,
  title={From local SGD to local fixed-point methods for federated learning},
  author={Malinovskiy, Grigory and Kovalev, Dmitry and Gasanov, Elnur and Condat, Laurent and Richtarik, Peter},
  booktitle={International Conference on Machine Learning},
  pages={6692--6701},
  year={2020},
  organization={PMLR}
}

@article{inconsistency3,
  title={Tackling the objective inconsistency problem in heterogeneous federated optimization},
  author={Wang, Jianyu and Liu, Qinghua and Liang, Hao and Joshi, Gauri and Poor, H Vincent},
  journal={Advances in neural information processing systems},
  volume={33},
  pages={7611--7623},
  year={2020}
}

@article{feddc,
  title={FedDC: Federated Learning with Non-IID Data via Local Drift Decoupling and Correction},
  author={Gao, Liang and Fu, Huazhu and Li, Li and Chen, Yingwen and Xu, Ming and Xu, Cheng-Zhong},
  journal={arXiv preprint arXiv:2203.11751},
  year={2022}
}

@inproceedings{fedadc,
  title={FedADC: Accelerated federated learning with drift control},
  author={Ozfatura, Emre and Ozfatura, Kerem and G{\"u}nd{\"u}z, Deniz},
  booktitle={2021 IEEE International Symposium on Information Theory (ISIT)},
  pages={467--472},
  year={2021},
  organization={IEEE}
}

@article{fedopt,
  title={FedOpt: Towards communication efficiency and privacy preservation in federated learning},
  author={Asad, Muhammad and Moustafa, Ahmed and Ito, Takayuki},
  journal={Applied Sciences},
  volume={10},
  number={8},
  pages={2864},
  year={2020},
  publisher={Multidisciplinary Digital Publishing Institute}
}

@article{liu2023enhance,
  title={Enhance Local Consistency in Federated Learning: A Multi-Step Inertial Momentum Approach},
  author={Liu, Yixing and Sun, Yan and Ding, Zhengtao and Shen, Li and Liu, Bo and Tao, Dacheng},
  journal={arXiv preprint arXiv:2302.05726},
  year={2023}
}

@inproceedings{fedproto,
  title={FedProto: Federated Prototype Learning across Heterogeneous Clients},
  author={Tan, Yue and Long, Guodong and Liu, Lu and Zhou, Tianyi and Lu, Qinghua and Jiang, Jing and Zhang, Chengqi},
  booktitle={AAAI Conference on Artificial Intelligence},
  volume={1},
  year={2022}
}

@inproceedings{hardt2016train,
  title={Train faster, generalize better: Stability of stochastic gradient descent},
  author={Hardt, Moritz and Recht, Ben and Singer, Yoram},
  booktitle={International conference on machine learning},
  pages={1225--1234},
  year={2016},
  organization={PMLR}
}

@inproceedings{zhang2022stability,
  title={Stability of sgd: Tightness analysis and improved bounds},
  author={Zhang, Yikai and Zhang, Wenjia and Bald, Sammy and Pingali, Vamsi and Chen, Chao and Goswami, Mayank},
  booktitle={Uncertainty in Artificial Intelligence},
  pages={2364--2373},
  year={2022},
  organization={PMLR}
}

@article{zhou2021towards,
  title={Towards understanding why lookahead generalizes better than sgd and beyond},
  author={Zhou, Pan and Yan, Hanshu and Yuan, Xiaotong and Feng, Jiashi and Yan, Shuicheng},
  journal={Advances in Neural Information Processing Systems},
  volume={34},
  pages={27290--27304},
  year={2021}
}

@article{neyshabur2017pac,
  title={A pac-bayesian approach to spectrally-normalized margin bounds for neural networks},
  author={Neyshabur, Behnam and Bhojanapalli, Srinadh and Srebro, Nathan},
  journal={arXiv preprint arXiv:1707.09564},
  year={2017}
}

@article{reisizadeh2020robust,
  title={Robust federated learning: The case of affine distribution shifts},
  author={Reisizadeh, Amirhossein and Farnia, Farzan and Pedarsani, Ramtin and Jadbabaie, Ali},
  journal={Advances in Neural Information Processing Systems},
  volume={33},
  pages={21554--21565},
  year={2020}
}

@article{bartlett2017spectrally,
  title={Spectrally-normalized margin bounds for neural networks},
  author={Bartlett, Peter L and Foster, Dylan J and Telgarsky, Matus J},
  journal={Advances in neural information processing systems},
  volume={30},
  year={2017}
}

@article{farnia2018generalizable,
  title={Generalizable adversarial training via spectral normalization},
  author={Farnia, Farzan and Zhang, Jesse M and Tse, David},
  journal={arXiv preprint arXiv:1811.07457},
  year={2018}
}

@article{shi2021fed,
  title={Fed-ensemble: Improving generalization through model ensembling in federated learning},
  author={Shi, Naichen and Lai, Fan and Kontar, Raed Al and Chowdhury, Mosharaf},
  journal={arXiv preprint arXiv:2107.10663},
  year={2021}
}

@inproceedings{yagli2020information,
  title={Information-theoretic bounds on the generalization error and privacy leakage in federated learning},
  author={Yagli, Semih and Dytso, Alex and Poor, H Vincent},
  booktitle={2020 IEEE 21st International Workshop on Signal Processing Advances in Wireless Communications (SPAWC)},
  pages={1--5},
  year={2020},
  organization={IEEE}
}

@article{cifar100,
  title={Learning multiple layers of features from tiny images},
  author={Krizhevsky, Alex and Hinton, Geoffrey and others},
  year={2009},
  publisher={Citeseer}
}

@article{dirichlet,
  title={Measuring the effects of non-identical data distribution for federated visual classification},
  author={Hsu, Tzu-Ming Harry and Qi, Hang and Brown, Matthew},
  journal={arXiv preprint arXiv:1909.06335},
  year={2019}
}

@inproceedings{resnet,
  title={Deep residual learning for image recognition},
  author={He, Kaiming and Zhang, Xiangyu and Ren, Shaoqing and Sun, Jian},
  booktitle={Proceedings of the IEEE conference on computer vision and pattern recognition},
  pages={770--778},
  year={2016}
}

@article{simonyan2014very,
  title={Very deep convolutional networks for large-scale image recognition},
  author={Simonyan, Karen and Zisserman, Andrew},
  journal={arXiv preprint arXiv:1409.1556},
  year={2014}
}

@inproceedings{gn_bn,
  title={The non-iid data quagmire of decentralized machine learning},
  author={Hsieh, Kevin and Phanishayee, Amar and Mutlu, Onur and Gibbons, Phillip},
  booktitle={International Conference on Machine Learning},
  pages={4387--4398},
  year={2020},
  organization={PMLR}
}

@inproceedings{improving_sam,
  title={Improving generalization in federated learning by seeking flat minima},
  author={Caldarola, Debora and Caputo, Barbara and Ciccone, Marco},
  booktitle={Computer Vision--ECCV 2022: 17th European Conference, Tel Aviv, Israel, October 23--27, 2022, Proceedings, Part XXIII},
  pages={654--672},
  year={2022},
  organization={Springer}
}

@inproceedings{sun2023dynamic,
  title={Dynamic regularized sharpness aware minimization in federated learning: Approaching global consistency and smooth landscape},
  author={Sun, Yan and Shen, Li and Chen, Shixiang and Ding, Liang and Tao, Dacheng},
  booktitle={International conference on machine learning},
  pages={32991--33013},
  year={2023},
  organization={PMLR}
}

@article{shi2023improving,
  title={Improving the model consistency of decentralized federated learning},
  author={Shi, Yifan and Shen, Li and Wei, Kang and Sun, Yan and Yuan, Bo and Wang, Xueqian and Tao, Dacheng},
  journal={arXiv preprint arXiv:2302.04083},
  year={2023}
}

@article{sun2023efficient,
  title={Efficient federated learning via local adaptive amended optimizer with linear speedup},
  author={Sun, Yan and Shen, Li and Sun, Hao and Ding, Liang and Tao, Dacheng},
  journal={IEEE Transactions on Pattern Analysis and Machine Intelligence},
  volume={45},
  number={12},
  pages={14453--14464},
  year={2023},
  publisher={IEEE}
}

@article{sun2023understanding,
  title={Understanding how consistency works in federated learning via stage-wise relaxed initialization},
  author={Sun, Yan and Shen, Li and Tao, Dacheng},
  journal={Advances in Neural Information Processing Systems},
  volume={36},
  pages={80543--80574},
  year={2023}
}

@inproceedings{vaswani2019fast,
  title={Fast and faster convergence of sgd for over-parameterized models and an accelerated perceptron},
  author={Vaswani, Sharan and Bach, Francis and Schmidt, Mark},
  booktitle={The 22nd international conference on artificial intelligence and statistics},
  pages={1195--1204},
  year={2019},
  organization={PMLR}
}

@inproceedings{kim2021lipschitz,
  title={The lipschitz constant of self-attention},
  author={Kim, Hyunjik and Papamakarios, George and Mnih, Andriy},
  booktitle={International Conference on Machine Learning},
  pages={5562--5571},
  year={2021},
  organization={PMLR}
}

@inproceedings{mai2021stability,
  title={Stability and convergence of stochastic gradient clipping: Beyond lipschitz continuity and smoothness},
  author={Mai, Vien V and Johansson, Mikael},
  booktitle={International Conference on Machine Learning},
  pages={7325--7335},
  year={2021},
  organization={PMLR}
}

@inproceedings{das2023beyond,
  title={Beyond uniform lipschitz condition in differentially private optimization},
  author={Das, Rudrajit and Kale, Satyen and Xu, Zheng and Zhang, Tong and Sanghavi, Sujay},
  booktitle={International Conference on Machine Learning},
  pages={7066--7101},
  year={2023},
  organization={PMLR}
}

@article{patel2022gradient,
  title={Gradient descent in the absence of global Lipschitz continuity of the gradients: Convergence, divergence and limitations of its continuous approximation},
  author={Patel, Vivak and Berahas, Albert S},
  journal={arXiv preprint arXiv:2210.02418},
  year={2022}
}

@article{sun2024fedpd,
  title={A-FedPD: aligning dual-drift is all federated primal-dual learning needs},
  author={Sun, Yan and Shen, Li and Tao, Dacheng},
  journal={Advances in Neural Information Processing Systems},
  volume={37},
  pages={85742--85777},
  year={2024}
}

@inproceedings{caldarola2023window,
  title={Window-based Model Averaging Improves Generalization in Heterogeneous Federated Learning},
  author={Caldarola, Debora and Caputo, Barbara and Ciccone, Marco},
  booktitle={Proceedings of the IEEE/CVF International Conference on Computer Vision},
  pages={2263--2271},
  year={2023}
}

@article{wang2022convergence,
  title={On the convergence of SGD under the over-parameter setting},
  author={Wang, Yiwei and Liu, Wei and Wang, Baoxiang and others},
  year={2022}
}

@inproceedings{karzand2019maximin,
  title={Maximin Active Learning with Data-Dependent Norms},
  author={Karzand, Mina and Nowak, Robert D},
  booktitle={2019 57th Annual Allerton Conference on Communication, Control, and Computing (Allerton)},
  pages={871--878},
  year={2019},
  organization={IEEE}
}

\end{document}